\documentclass[journal]{IEEEtran}

\usepackage{cite}
\usepackage{amsmath,amssymb,amsfonts}
\usepackage{amsthm}
\usepackage{xcolor}
\usepackage[pdftex]{graphicx}
\usepackage{algorithm,algorithmic}
\usepackage{array}
\usepackage[caption=false,font=footnotesize]{subfig}
\usepackage{stfloats}
\usepackage{url}
\usepackage{booktabs,longtable}
\usepackage{tabularx}

\newtheorem{proposition}{Proposition}
\newtheorem{corollary}{Corollary}
\newtheorem{problem}{Problem}

\newcommand{\resultframe}[2]{%
\IfFileExists{#1}{\includegraphics[width=0.47\textwidth]{#1}}{%
\fbox{\parbox[c][38mm][c]{0.45\textwidth}{\centering#2}}}}

\begin{document}

\title{
    Learning Altruistic Collaboration in Heterogeneous Multi-Team Systems
    \thanks{
        $^1$R. Karam and R. Lin are with the Samueli School of Engineering, University of California, Irvine, Irvine, CA, 92697, USA. Email: {\tt\small \{rwkaram, rlin10\}@uci.edu}.
    }
    \thanks{
        $^2$B. A. Butler is with the Oklahoma State University, Stillwater, OK, 74078. Email: {\tt\small brooks.butler@okstate.edu}
    }
    \thanks{
        $^3$M. Egerstedt is with the University of North Carolina at Chapel Hill, Chapel Hill, NC, 27599, USA. Email: {\tt\small magnus@unc.edu}.
    }
}

\author{\IEEEauthorblockN{Riwa Karam$^1$, Ruoyu Lin$^1$, Brooks A. Butler$^2$, Magnus Egerstedt$^3$}}

\maketitle

\begin{abstract}
    This paper studies heterogeneous multi-team collaboration through dynamic robot allocation, where the robots themselves are treated as transferable resources. Leveraging Hamilton’s rule from ecology as an altruistic decision-making mechanism, we propose a multi-team collaborative resource allocation framework with heterogeneous capabilities, transfer costs, and capability-dependent contributions. The {resulting general heterogeneous resource assignment} problem is combinatorial and is shown to be NP-hard. To address scalability, we develop a graph neural network policy under centralized training and decentralized execution that approximates the altruistic allocations based on Hamilton’s rule. The model operates over the team interaction graph and predicts robot-level transfer decisions and subsequent robot-to-team assignments. The proposed approach is {evaluated} in a fire-fighting scenario {against exact optimization and multi-step simulations, quantifying allocation quality, decision computation, and fire suppression. $5$ paired physical testbed trials with $15$ robots compare learned reallocation with fixed team membership; the GNN reaches the modeled fire stopping threshold in all $5$ trials, compared with $2$ for static allocation}.
\end{abstract}


\section{Introduction} \label{sec:introduction}

Multi-robot systems hold out promise of providing a scalable and robust approach for complex tasks such as disaster response~\cite{murphy2017disaster}, distributed sensing~\cite{cortes2004coverage}, and environmental monitoring~\cite{schwager2009decentralized}. A fundamental challenge in these systems is \textit{resource and task allocation}, where different resources and/or tasks must be assigned to different agents to optimize overall system performance~{\cite{gerkey2004formal}}. In many applications, these decisions must be made dynamically and in a decentralized manner while accounting for limited resources, spatial constraints, and time-varying environments~{\cite{korsah2013comprehensive}}.

{When robots are treated as identical resources}, the performance of a team {can be modeled as a function of} the number of assigned robots~{\cite{karam2025resource}}. However, heterogeneity inherently increases the complexity of each agent's contributions to its team, especially when mission success relies on collaborative interactions \cite{karam2026collaboration}. In heterogeneous systems, robots may possess complementary sensing, actuation, or computational capabilities, and the contribution of a robot depends on the composition of the team it joins. As a result, team performance may become set-dependent and generally non-separable, introducing combinatorial coupling across agents.

Graph neural networks (GNNs) provide a framework for learning policies over relational structures {\cite{scarselli2009graph}. In} multi-robot systems, {graph-based policies can learn decentralized decisions from centralized supervision~\cite{tolstaya2020learning}}. In this article, we are interested in how heterogeneous robots with complementary capabilities contribute to the team performance through joint effort.

Biological systems provide {examples of interactions} in multi-agent settings~\cite{bonabeau1999swarm}, with evolutionary perspectives offering fundamental explanations for cooperation and altruistic behaviors~\cite{west2007evolutionary}. In particular, \textit{Hamilton's rule}~\cite{hamilton1964genetical} characterizes altruistic behavior through a trade-off between cost and benefit weighted by genetic relatedness. This principle has been adopted in multi-robot systems to regulate inter-team resource exchanges via a local decision rule in \cite{karam2025resource}, {where a transfer happens when the recipient's benefit exceeds the donor's loss in mission value}. In~\cite{karam2025resource}, the robots are treated as transferable resources, which yields globally optimal allocations under homogeneous assumptions and diminishing returns.

In this paper, we extend the altruistic collaboration framework from \cite{karam2025resource} to heterogeneous multi-team systems. Heterogeneity is modeled through capability vectors and set-dependent mission evaluation functions, and introduces explicit transfer costs associated with robot migration. The {general heterogeneous assignment} problem is NP-hard, {motivating a learned approximation of its one-step allocation decisions}. We consider an event-triggered hybrid system with continuous-time mission dynamics and discrete-time robot exchanges, and aim to maximize a global mission evaluation function in a decentralized setting. To enable scalable decision-making, the problem is formulated as a structured prediction task over graphs, and a GNN policy is developed under centralized training and decentralized execution (CTDE). {Centralized training uses globally optimized one-step assignments as supervision, while during decentralized execution, each robot scores potential transfers using local features and team information exchanged over the interaction graph. The policy} operates over the team interaction graph and predicts robot-level transfer decisions, enabling scalable approximation of the underlying combinatorial optimization problem.
 
The main contributions of this paper are:
\begin{enumerate}
    \item An altruistic multi-team allocation framework {that integrates} heterogeneous {robot capabilities}, {set-dependent} mission evaluation functions, and transfer costs {with robot-specific Hamilton-admissibility};
    \item {A GNN policy that combines team-level message passing, robot-conditioned action scoring, and Hamilton-admissible masking to approximate one-step allocation decisions} under CTDE.
\end{enumerate}
{These contributions are supported by a complexity analysis and evaluated in a fire-fighting application through allocation-quality assessments, runtime comparisons, larger-scale simulations, and paired physical testbed experiments.}

The remainder of this paper is organized as follows. Section~{\ref{sec:related_work} discusses related work}. Section~\ref{sec:formulation} reviews the homogeneous altruistic collaboration framework and presents the heterogeneous extension of it. Section~\ref{sec:complexity} analyzes and compares the complexity of the homogeneous and heterogeneous resource allocation problems. Section~\ref{sec:learning} presents the learning-based approach. Section~\ref{sec:application} instantiates the theoretical altruistic collaboration model {in} a fire-fighting scenario. Section~\ref{sec:results} provides experimental results, and Section~\ref{sec:conclusion} concludes the paper.

\section{{Related Work}} \label{sec:related_work}

{The proposed framework brings together heterogeneous robot allocation, graph-based learning, and altruistic inter-team transfers. We review these three themes in turn to situate how robot capabilities are represented, how allocation decisions are learned, and how the benefit and cost of collaboration guide transfers between teams.}

\subsection{{Robot Allocation and Heterogeneous Capabilities}}

{Multi-robot allocation methods differ in how assignments are negotiated and how robot contributions are represented. Formal and comprehensive taxonomies organize this range of problems~\cite{gerkey2004formal,korsah2013comprehensive}. Distributed auctions provide a mechanism for negotiating assignments~\cite{zavlanos2008distributed}, with consensus-based auctions supporting robust decentralized allocation~\cite{choi2009consensus}. When the operating conditions change, task-allocation games address evolving assignments~\cite{park2021multi}, while distributed resource-allocation methods account for time-varying networks~\cite{lu2022distributed} and decentralized resource control~\cite{anussornnitisarn2005decentralized}.}

{For heterogeneous robots, the representation of capabilities is central to the allocation decision, such as mobility and sensing capabilities~\cite{lin2025heterogeneous}. Feature-based capabilities can enter optimization constraints for energy-aware task prioritization and execution~\cite{notomista2021resilient}. At the population level, species-specific transition rates can be optimized to achieve task-dependent trait distributions~\cite{prorok2017diversity}; STRATA accounts for trait variability when allocating heterogeneous agents to satisfy task requirements~\cite{ravichandar2020strata}. When capabilities are unknown or evolve during execution, task specialization can instead adapt to observed progress~\cite{emam2020adaptive}. These approaches motivate an explicit account of what each robot contributes, whether through capability constraints, trait requirements, or execution feedback.}

{Coalition formation considers how robots are grouped to carry out tasks~\cite{vig2006coalition}. GRAPE-S uses task/service-dependent preferences to form distributed coalitions with multiple service types, and Pair-GRAPE-S also considers coordinated reassignments~\cite{diehl2023grapes}. This formulation specifies the required counts for each service at a task. An optimization-based alternative repeatedly clusters robots with tasks and optimizes assignments within each cluster~\cite{martin2026iterative}.}

\subsection{{Graph Learning for Coordination and Optimization}}

{Learning provides a complementary approach to combinatorial decision-making~\cite{bengio2021machine}, including methods that predict solutions from graph representations~\cite{khalil2017learning}. GNNs model graph-structured data~\cite{scarselli2009graph}, while graph networks provide a broader formulation of relational representations~\cite{battaglia2018relational}. Such representations have been used in routing~\cite{kool2019attention}, job-shop scheduling~\cite{zhang2020learning}, and physics-based inference and control~\cite{sanchez2018graph}. Their relevance to multi-robot systems lies in representing both individual agents and the interactions that shape their decisions.}

{Learning coordinated behavior and executing it locally~can be approached through several mechanisms. Monotonic value-function factorization supports multi-agent reinforcement learning~\cite{rashid2020monotonic}, whereas expert imitation has been used to~learn communication and decisions for multi-robot path planning~\cite{li2020graph}. Local swarm controllers can likewise be learned from centralized supervision~\cite{tolstaya2020learning}. These precedents motivate the use of centralized training with graph-based information exchange at execution, while leaving the prediction target and the admissible actions to be specified by the allocation problem.}

{The task model also shapes what a learned policy must represent. Graph attention combined with expert imitation addresses temporal and spatial constraints in multi-robot scheduling~\cite{wang2020scheduling}, while shared hypernetwork policies can condition decisions on individual and collective capabilities~\cite{fu2025cash}. For dynamic coalition formation, MAPPO has been combined with spatial action maps, intention sharing, and task revision in an environment of homogeneous robots, evolving tasks, and local observations~\cite{bezerra2025learningdcfmrta}. These formulations offer distinct precedents for constrained decisions, capability-aware policies, and adaptation to changing tasks. Our prediction target is instead each robot's next team under a composition-dependent mission objective and a robot-specific Hamilton-admissible action set.}

\subsection{{Altruistic Transfers and the Present Scope}}

{Biological models of collective behavior~\cite{bonabeau1999swarm} and the evolution of cooperation~\cite{west2007evolutionary} provide another perspective on coordination: an individual action can be assessed through its cost to the donor and its benefit to others. Hamilton's rule expresses this tradeoff through relatedness-weighted benefit~\cite{hamilton1964genetical}. In multi-team allocation, it has been used to regulate exchanges of robots treated as transferable resources, with global allocation optimality established under homogeneous, diminishing-returns assumptions~\cite{karam2025resource}. It has also been used to model inter-agent altruistic decisions~\cite{butler2026hamiltonTAC,butler2026altsafety}.}

{These aforementioned research directions inform our study of altruistic allocation among existing teams of heterogeneous robots. We consider composition-dependent mission values and robot-specific migration costs, using Hamilton's rule to identify admissible transfers and a graph-based policy to learn allocation decisions. This connects capability-aware allocation and learned coordination with an explicit assessment of the benefit and loss of inter-team transfers.}

\section{Problem Formulation} \label{sec:formulation}

In this section, the multi-team robot allocation problem underlying altruistic collaboration is formalized. We begin by reviewing the homogeneous framework introduced in~\cite{karam2025resource}, where robots are treated as identical resources and Hamilton’s rule provides a local mechanism for regulating inter-team transfers. This framework is then extended to the heterogeneous setting, where robots possess distinct capabilities and team performance depends on the composition of assigned robots rather than their count alone.

\subsection{Homogeneous Altruistic Framework} \label{ssec:homogeneous}

In the homogeneous setting of~\cite{karam2025resource}, multiple teams of identical robots operate in distinct regions, and robots are treated as transferable resources whose allocation determines team performance. Interactions between teams are represented by an undirected graph \( G = (V, E) \), where \( V \) is the set of teams with cardinality \( M = |V| \), and \( E \subseteq V \times V \) encodes the pairwise team connections.

Each team \( v \in V \) is associated with a mission evaluation function \( F_v(n_v) : \mathbb{N} \to \mathbb{R} \), where \( n_v \) denotes the number of robots assigned to team \( v \). The functions \( F_v \) are assumed to satisfy the following assumptions:
\begin{itemize}
    \item \textit{Strict monotonicity:} \( F_v(n_v + 1) > F_v(n_v) \),
    \item \textit{Diminishing returns:}
    \[
        F_v(n_v + 1) - F_v(n_v) \le F_v(n_v) - F_v(n_v - 1).
    \]
\end{itemize}
The genetic relatedness term is modeled through a ratio of mission-importance weights, chosen to reflect differences in team priority, as
\begin{equation} \label{eq:relatedness}
    r_{ij} = \frac{w_j}{w_i},
\end{equation}
where \( w_i, w_j > 0 \) denote the relative importance assigned to teams \( i \) and \( j \), respectively. The marginal benefit and cost are defined as
\begin{equation} \label{eq:homogeneous_benefit_cost}
    \begin{aligned}
        B_j(n_j) &:= F_j(n_j + 1) - F_j(n_j), \\
        C_i(n_i) &:= F_i(n_i) - F_i(n_i - 1).
    \end{aligned}
\end{equation}
{A transfer from team $i$ to team $j$ is Hamilton-admissible when $r_{ij}B_j(n_j)>C_i(n_i)$.} A key structural property established in~\cite{karam2025resource} is the \emph{pairwise uni-directionality} of admissible transfers. Under the above assumptions, the conditions
\begin{equation} \label{eq:uni_directionality_homogeneous}
    r_{ij} B_j(n_j) > C_i(n_i)
    \quad \text{and} \quad
    r_{ji} B_i(n_i) > C_j(n_j)
\end{equation}
cannot hold simultaneously. {The local bidding mechanism in~\cite{karam2025resource} resolves competing admissible transfers.} Each team selects at most one outgoing and one incoming transfer based on the highest positive gain. These local decisions collectively aim to maximize the global objective
\begin{equation} \label{eq:global_func}
    \mathcal{G}(n_1, \dots, n_M) := \sum_{v \in V} w_v F_v(n_v).
\end{equation}
{Under the assumptions and exchange procedure in~\cite{karam2025resource}, executed transfers strictly increase the global objective, the process terminates in a finite number of steps, and the final allocation is globally optimal. These guarantees apply to the homogeneous framework and do not extend directly to the heterogeneous setting.}

To express the problem at the robot level, let \( \mathcal{R} = \{ 1, \dots, N \} \) and the assignment variable be
\begin{equation} \label{eq:one-hot}
    x_{r,v} =
    \begin{cases}
        1, & \text{if robot \( r \in \mathcal{R} \) is assigned to team \( v \in V \),} \\
        0, & \text{otherwise.}
    \end{cases}
\end{equation}
{Here, $N$ is the total number of robots.} Since robots are homogeneous in \cite{karam2025resource}, the mission evaluation function of team \( v \) depends only on its team size, given by $n_v = \sum_{r \in \mathcal{R}} x_{r,v},$ $\forall v \in V$. The global objective can then also be written as
\begin{equation} \label{eq:global_func_assignment}
    \mathcal{G}(X) = \sum_{v \in V} w_v F_v \bigg( \sum_{r \in \mathcal{R}} x_{r,v} \bigg),
\end{equation}
where \( X = [x_{r,v}]_{r \in \mathcal{R},\, v \in V} \in \{ 0, 1 \}^{N \times M} \) denotes the assignment matrix. {The corresponding global assignment problem~is}
\begin{equation} \label{eq:opt_homogeneous_centralized}
    \begin{aligned}
        \max_{X} \quad 
        & \mathcal{G}(X) \\
        \text{s.t.} \quad 
        & \sum_{v \in V} x_{r,v} = 1, \quad \forall r \in \mathcal{R}, \\
        & \sum_{r \in \mathcal{R}} x_{r,v} \ge 1, \quad \forall v \in V, \\
        & x_{r,v} \in \{0,1\}, \quad \forall r \in \mathcal{R}, \, \forall v \in V.
    \end{aligned}
\end{equation}
The constraints in Equation~\eqref{eq:opt_homogeneous_centralized} {ensure that each robot belongs to exactly one team and every team contains at least one robot}.

\subsection{Heterogeneous Altruistic Collaboration} \label{ssec:heterogeneous}

To extend the framework in \cite{karam2025resource}, a heterogeneous multi-team setting is considered in which robots possess different capabilities. Unlike the homogeneous case, where team performance depends only on the number of assigned robots, the contribution of each robot now depends on its capabilities and the composition of the team it joins. As a result, the mission evaluation functions become set-dependent, fundamentally altering the structure of the assignment problem. Robot heterogeneity is encoded through capability vectors{,} and transfer costs {account for} robot migration between teams.

Let each robot \( r \) be characterized by a capability vector \( \mathbf{\kappa}_r \in \{0, 1\}^Q \), where each entry indicates the presence~(\(1\)) or absence~(\(0\)) of a specific capability; \( Q \) denotes the number of capability types. For each team \( v \in V \), let the set of robots assigned to team \( v \) at iteration \( k \) be defined as
\[
    \mathcal{S}_v^k := \{ r \in \mathcal{R} \mid x^k_{r,v} = 1 \}, \,\, \mathcal{S}_v^k \subset \mathcal{R}.
\]
An example of the heterogeneous multi-team system is illustrated in Figure~\ref{fig:heterogeneous_example}, where teams, heterogeneous robots, and Hamilton-admissible transfer directions are shown. {The Hamilton-admissible directions} in the figure are robot-specific; in the homogeneous setting {of} Section~\ref{ssec:homogeneous}, {admissibility does not depend on} which robot is chosen.

{The one-team-per-robot} constraint {$\sum_{v\in V}x_{r,v}^k=1$} {ensures that} each robot {has} a unique {current} team. The current team of robot \( r \) (i.e., at iteration \( k \)) is defined as
\begin{equation} \nonumber
    \mathrm{cur}(r) = v \quad \text{such that} \quad x^k_{r,v}=1,
\end{equation}
and its chosen next team (i.e., at iteration \( k + 1 \)) is defined as
\begin{equation} \nonumber
    \mathrm{next}(r) = v \quad \text{such that} \quad x^{k+1}_{r,v} = 1.
\end{equation}
Let the binary transfer variables \( y^{k+1}_{r,i \to j} \), indicating whether robot \( r \) moves from team \( i \) to team \( j \), be
\begin{equation} \label{eq:binary_transfer}
    y^{k+1}_{r,i \to j} =
    \begin{cases}
        1, & \text{if robot $r \in \mathcal{R}$ moves from team $i$} \\
           & \text{to team $j$ at iteration $k+1$,} \\
        0, & \text{otherwise,}
    \end{cases}
\end{equation}
for all $r \in \mathcal{R}$ and $(i,j) \in E$. Note that \( \mathrm{next}(r) \neq \mathrm{cur}(r) \) happens when \( y^{k+1}_{r, \mathrm{cur}(r) \to \mathrm{next}(r)} = 1 \); otherwise, \( \mathrm{next}(r) = \mathrm{cur}(r) \), meaning robot \( r \) was not transferred at iteration \( k \). In particular, we get \( r \in \mathcal{S}^k_{\mathrm{cur}(r)} \) and \( r \in \mathcal{S}^{k+1}_{\mathrm{next}(r)} \). Let \(Y^{k+1} = \{y^{k+1}_{r,i\to j}\}\) then denote the set of binary transfer decision variables.

\begin{figure}[!t]
    \centering
    \includegraphics[width=1\columnwidth]{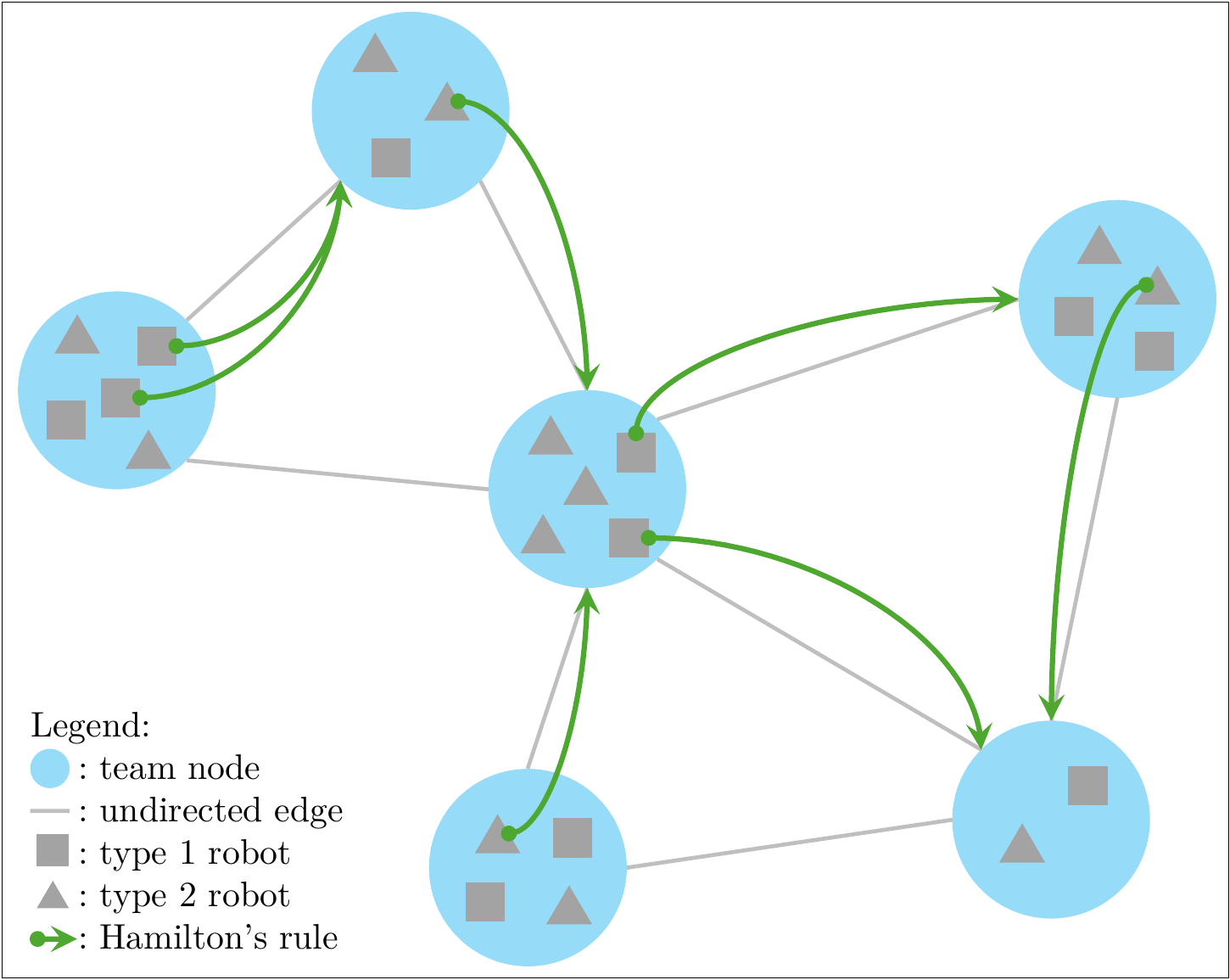}
    \caption{Example of a heterogeneous multi-team system where teams, depicted as blue nodes, are connected through an interaction graph with undirected edges and composed of robots with different capabilities. Robots of different types (type~1 and type~2, shown as squares and triangles, respectively) correspond to distinct capability classes as defined by their capability vectors. Directed edges (green arrows) represent candidate inter-team robot transfers that satisfy Hamilton-admissibility.}
    \label{fig:heterogeneous_example}
\end{figure}

{Unlike the model in~\cite{karam2025resource}, reviewed in Section~\ref{ssec:homogeneous}, a heterogeneous team's performance depends on its members' capabilities. We therefore define the mission evaluation function for each team \(v\in V\) over robot sets:}
\begin{equation} \label{eq:mission_eval_func}
    \mathcal{F}_v : \{ \mathcal{S} \mid \mathcal{S} \subseteq \mathcal{R}\} \to \mathbb{R},
\end{equation}
evaluated at the assigned robot set \( \mathcal{S}_v^k \). In particular, \( n_v^k = |\mathcal{S}_v^k| \), where \( |\mathcal{S}_v^k| \) denotes the cardinality of \( \mathcal{S}_v^k \). This formulation captures the dependence of team performance on the joint capabilities of its members, introducing collaboration as defined in \cite{karam2026collaboration}, thereby leading to combinatorial interactions between robots.

Since a transfer cost is introduced, we define the distance needed for robot \( r \) to travel from its current team \( \text{cur}(r) \) to any team \( v \in V \) to be \( d_{r, \, \mathrm{cur}(r) \to v} \in \mathbb{R^+} \). Hence, the transfer cost becomes
\begin{equation} \label{eq:transfer_cost}
    \mathcal{C}(X^k, X^{k+1}) := \sum_{r \in \mathcal{R}} \sum_{v \in V} x_{r, v}^{k+1} \, c_{r, v}(x^k_{r, v}),
\end{equation}
where \( c_{r, v}: \{ 0, 1 \} \to \mathbb{R}^+ \) can depend on the distance \( d_{r, \, \mathrm{cur}(r) \to v} \) and the speed of robot \( r \) denoted as \( s_r \in \mathbb{R^+} \).

{The mission evaluation function and transfer cost in Equations~\eqref{eq:mission_eval_func} and~\eqref{eq:transfer_cost} can be instantiated in different ways, even for the same mission. Their definitions here provide a general formulation, with the notation summarized in Table~\ref{tab:notation}.}

\begin{table}[!b]
    \centering
    \renewcommand{\arraystretch}{1.5}
    \caption{Problem Formulation Notation}
    \label{tab:notation}
    \begin{tabular}{c|l}
        \hline
        \textbf{Symbol} & \textbf{Description} \\
        \hline
        \( V \) & Set of teams \\
        \hline
        \( E \) & Set of team edges \\
        \hline
        \( G = (V, E) \) & Team interaction graph \\
        \hline
        \( \mathcal{R} \) & Set of robots \\
        \hline
        \( \mathbf{\kappa}_r \) & Capability vector of robot \( r \) \\
        \hline
        \( s_r \) & Speed of robot \( r \) \\
        \hline
        \( \mathcal{S}_v^k \) & Set of robots assigned to team \( v \) at iteration \( k \) \\
        \hline
        \( \mathcal{F}_v(\mathcal{S}_v^k) \) & Mission evaluation function \\
        \hline
        \( X^k \) & Assignment matrix at iteration \( k \) \\
        \hline
        \( \mathcal{G}(X^k) \) & Global mission evaluation function \\
        \hline
        \( y^{k+1}_{r,i\to j} \) & Transfer decision variable \\
        \hline
        \( h^{k+1}_{r,i\to j} \) & Hamilton-admissible indicator \\
        \hline
        \( \mathcal{C}(X^k,X^{k+1}) \) & Global transfer cost \\
        \hline
        \( \lambda \) & Regularization parameter \\
        \hline
    \end{tabular}
\end{table}

Since the goal of the multi-team system is to optimally and successfully achieve the assigned missions, we wish to optimize the system efficiency in a way that the global mission evaluation function is maximized and the global transfer cost is minimized. The following optimization problem reflects those needs:
\begin{equation} \label{eq:opt_heterogeneous_centralized}
    \begin{aligned}
        \max_{X} \quad 
        & \mathcal{G}(X) - \lambda \, \mathcal{C}(X^0, X) \\
        \text{s.t.} \quad 
        & \sum_{v \in V} x_{r,v} = 1, \quad \forall r \in \mathcal{R}, \\
        & \sum_{r \in \mathcal{R}} x_{r,v} \ge 1, \quad \forall v \in V, \\
        & x_{r,v} \in \{0,1\}, \quad \forall r \in \mathcal{R}, \, \forall v \in V,
    \end{aligned}
\end{equation}
with \( \mathcal{G}(X^k) := \sum_{v \in V} w_v \mathcal{F}_v(\mathcal{S}_v^k) \) being the heterogeneous global mission evaluation function, the regularization parameter \( \lambda > 0 \) representing the balance of the trade-off between minimizing the costs and maximizing mission performance, and \( X^0 \) being a random initial assignment matrix. Note that, similar to Equation~\eqref{eq:opt_homogeneous_centralized}, the use of \( k \) is omitted since the optimization problem in Equation~\eqref{eq:opt_heterogeneous_centralized} presents the general assignment problem to be solved in the heterogeneous formulation, with {the same constraints as Equation~\eqref{eq:opt_homogeneous_centralized} but a different objective function}.

Next, the iterative, altruistic, and heterogeneous optimization problem that we wish to solve in this article is presented. Let the Hamilton-admissible indicator be
\begin{equation} \label{eq:hamilton_indicator}
    \begin{aligned}
        &h^{k+1}_{r, i \to j} := \mathbb{I} \left\{ r_{ij} \, B_{r,j}(\mathcal{S}_j^k) > C_{r,i}(\mathcal{S}_i^k) \right\}, \\
        &h^{k+1}_{r, i \to j}\in \{0, 1\}, \,\, \forall (i, j) \in E,
    \end{aligned}
\end{equation}
with
\begin{equation} \label{eq:heterogeneous_benefit_cost}
    \begin{aligned}
        B_{r, j}(\mathcal{S}_j^{k}) &:= \mathcal{F}_j(\mathcal{S}_j^{k} \, \cup \, \{r\} ) - \mathcal{F}_j(\mathcal{S}_j^{k}), \\
        C_{r, i}(\mathcal{S}_i^{k}) &:= \mathcal{F}_i( \mathcal{S}_i^{k} ) - \mathcal{F}_i( \mathcal{S}_i^{k} \, \setminus \, \{r\} ).
    \end{aligned}
\end{equation}
Note that, in the homogeneous setting, the benefit and cost factors consider the addition or subtraction of one robot (Equation~\eqref{eq:homogeneous_benefit_cost}), whereas in Equation~\eqref{eq:heterogeneous_benefit_cost}, they become set-dependent. The same relatedness factor is considered as in Equation~\eqref{eq:relatedness}. {For an isolated transfer at the current state, Hamilton-admissibility is equivalent to $w_j B_{r,j}-w_i C_{r,i}>0$, a positive marginal change in the unregularized weighted mission value. Thus, Hamilton’s rule admits an individual transfer only when its weighted benefit exceeds its weighted loss. This test evaluates each robot separately using the current team composition and does not include migration costs.}

{The following one-step problem selects a feasible transition from the current assignment using the heterogeneous mission objective and the cost of that transition:}
\begin{equation} \label{eq:opt_heterogeneous_iterative}
    \begin{aligned}
        \max_{X^{k+1},\,Y^{k+1}} \quad
        & \mathcal{G}(X^{k+1}) - \lambda \, \mathcal{C}(X^k, X^{k+1}) \\
        \text{s.t.} \quad
        & \sum_{v \in V} x^{k+1}_{r,v} = 1, \quad \forall r \in \mathcal{R}, \\
        & \sum_{r \in \mathcal{R}} x^{k+1}_{r,v} \ge 1, \quad \forall v \in V, \\
        & x^{k+1}_{r,v} \in \{0,1\}, \quad \forall r \in \mathcal{R}, \, \forall v \in V, \\
        & x^{k+1}_{r,j} = x^k_{r,j}
          - \sum_{l:(j,l)\in E} y^{k+1}_{r,j \to l}
          + \sum_{i:(i,j)\in E} y^{k+1}_{r,i \to j}, \\
        & \qquad \forall r \in \mathcal{R}, \, \forall j \in V, \\
        & y^{k+1}_{r,i \to j} \in \{0,1\}, 
          \quad \forall r \in \mathcal{R}, \, \forall (i,j) \in E, \\
        & y^{k+1}_{r,i \to j} \le x^k_{r,i}, 
          \quad \forall r \in \mathcal{R}, \, \forall (i,j) \in E, \\
        & y^{k+1}_{r,i \to j} \le h^{k+1}_{r,i \to j}, 
          \quad \forall r \in \mathcal{R}, \, \forall (i,j) \in E.
    \end{aligned}
\end{equation}
Unlike the homogeneous case, teams are not restricted to a single incoming or outgoing transfer per iteration. Instead, multiple robots may be transferred simultaneously, reflecting the heterogeneous contributions of individual robots. Given the heterogeneity of robots and the set-dependent evolution of the team missions, the uni-directionality property in Equation~\eqref{eq:uni_directionality_homogeneous} no longer applies. {The indicator in Equation~\eqref{eq:hamilton_indicator} is therefore robot-specific, distinguishing each robot's contribution to a transfer and its effect on the team objective.} {Equation~\eqref{eq:opt_heterogeneous_iterative} includes the assignment and nonempty-team constraints of Equations~\eqref{eq:opt_heterogeneous_centralized} and~\eqref{eq:opt_homogeneous_centralized}. The added transfer constraints ensure that each transferred robot moves from its current team to a single adjacent destination and satisfies Hamilton's rule.}

{Equation~\eqref{eq:opt_heterogeneous_iterative} evaluates the combined effect of the chosen transfers on mission value and migration cost. For a feasible current assignment, staying is also feasible: $X^{k+1}=X^k$ and $Y^{k+1}=0$ satisfy the constraints. An exact solution therefore achieves an objective value no lower than that of the unchanged allocation at the same decision state. This guarantee follows from the joint optimization, rather than from the individual Hamilton-admissibility checks alone.}

{The heterogeneous allocation framework developed above provides the basis for the multi-team collaboration problem considered in this paper, stated formally in Problem~\ref{pl:problem}.}

\begin{problem} \label{pl:problem}
Given a multi-team robot system with an interaction graph \( G \) and a set of heterogeneous robots \( \mathcal{R} \), where each robot is characterized by a capability vector \( \kappa_r \) and each team \( v \)'s performance is quantified by a set-dependent mission evaluation function \( \mathcal{F}_v \), the objective is to determine a sequence of local robot-to-team assignment decisions that maximizes the global mission performance \( \mathcal{G} \) while minimizing the global transfer cost \( \mathcal{C} \).
\end{problem}

{We approach Problem}~{\ref{pl:problem} through decentralized, Hamilton-admissible transfer decisions that approximate Equation}~{\eqref{eq:opt_heterogeneous_iterative}. This one-step objective specifies the decision made at each iteration; the global optimality scope is discussed in Section}~{\ref{ssec:optimization_scope}.}

\section{Complexity Analysis} \label{sec:complexity}

In this section, we analyze and compare the computational complexity of the optimization problems in Equation~\eqref{eq:opt_homogeneous_centralized}, the iterative process in~\cite{karam2025resource}, Equation~\eqref{eq:opt_heterogeneous_centralized}, and Equation~\eqref{eq:opt_heterogeneous_iterative}.

\subsection{Homogeneous Resource Allocation} \label{ssec:homogeneous_complexity}

In the homogeneous case, the feasible set is the discrete simplex
\begin{equation} \nonumber
    \left\{ \mathbf{n} \in (\mathbb{Z}^+)^M \,\middle|\, \sum_{v=1}^M n_v = N \right\},
\end{equation}
whose cardinality is \( \binom{N-1}{M-1} \).
A brute-force solution to Equation~\eqref{eq:opt_homogeneous_centralized} enumerates all feasible allocations and evaluates the global objective
\( \mathcal{G}(\mathbf{n}) = \sum_{v\in V} w_v F_v(n_v) \) for each.
Assuming evaluating \(\mathcal{G} \) takes \( O(M) \) time, the brute-force complexity is
\begin{equation} \nonumber
    O\!\left(M \binom{N-1}{M-1}\right),
\end{equation}
which grows combinatorially in \( N \) and \( M \).

The altruistic iterative process in~\cite{karam2025resource} restricts explicit allocation by using Hamilton’s rule and a bidding mechanism over \( G = (V, E) \). Per iteration, it performs:
\begin{itemize}
    \item Hamilton’s-rule checks and bid computations on each edge: \( O(|E|) \);
    \item selecting at most one outgoing and one incoming collaboration per team (implemented by scanning neighbors): \( O(|E|) \) total;
    \item updating team sizes and marginal values: \( O(M) \).
\end{itemize}
Hence the per-iteration complexity is \( O(|E| + M) = O(|E|) \).
Moreover, under the assumptions in~\cite{karam2025resource}, the global objective \( \mathcal{G} \) strictly increases whenever a transfer is executed, so the same allocation \( \mathbf{n} \) cannot be revisited. Because there are only \( \binom{N-1}{M-1} \) feasible allocations, the number of executed transfers is at most \( \binom{N-1}{M-1} - 1 \). Therefore, a loose worst-case runtime bound is
\begin{equation} \nonumber
    O \! \left( |E| \binom{N-1}{M-1} \right).
\end{equation}
In practice, far fewer transfers occur since the algorithm follows a monotone path to an optimum rather than exploring the full feasible set; the average runtime scales approximately as \( O(|V||E|) \).

\subsection{Heterogeneous Resource Allocation} \label{ssec:heterogeneous_complexity}

In the heterogeneous formulation (Equation~\eqref{eq:opt_heterogeneous_centralized}), the decision variable is the binary assignment matrix $X \in \{0,1\}^{N\times M}$. The number of feasible assignments is at most $M^N$ (each robot chooses one of $M$ teams), so brute force requires $O(M^N)$ objective evaluations, which is exponential in $N$.

\begin{proposition} \label{pr:proposition}
    The heterogeneous allocation problem in Equation~\eqref{eq:opt_heterogeneous_centralized} is NP-hard.
\end{proposition}
\begin{proof}
    We reduce from the \textit{Partition} problem \cite{garey2002computers}, which is NP-complete.
    Given a set of positive integers $\{a_1, \dots, a_N\}$ with total sum $A = \sum_{r=1}^N a_r$, the Partition problem asks whether there exists a subset $\mathcal{I} \subseteq \mathcal{R}$ such that $\sum_{r \in \mathcal{I}} a_r = \sum_{r \notin \mathcal{I}} a_r = {A}/{2}$. A special instance of the heterogeneous allocation problem is constructed as follows:
    \begin{itemize}
        \item There are $M = 2$ teams, denoted as $1$ and $2$.
        \item {Associate each robot $r\in\mathcal R$ with a positive integer parameter $a_r$ in the mission evaluation function. These are robot-indexed mission parameters, not entries of the binary capability vector.}
        \item The mission evaluation functions are defined as
        \[
            \mathcal{F}_v(\mathcal{S}_v) = -\left| \sum_{r \in \mathcal{S}_v} a_r - \frac{A}{2} \right|, \quad v \in \{1,2\}.
        \]
        \item The team weights are equal: $w_1 = w_2 = 1$.
        \item {The transfer-cost function is identically zero, with $\lambda>0$ unchanged. Equation~\eqref{eq:opt_heterogeneous_centralized} has no Hamilton-admissibility constraint.}
    \end{itemize}
    {Under this construction}, the objective in Equation~\eqref{eq:opt_heterogeneous_centralized} reduces to $\mathcal{G}(X) = \mathcal{F}_1(\mathcal{S}_1) + \mathcal{F}_2(\mathcal{S}_2)$. Since $\mathcal{S}_1$ and $\mathcal{S}_2$ form a partition of $\mathcal{R}$, we have $\sum_{r \in \mathcal{S}_1} a_r + \sum_{r \in \mathcal{S}_2} a_r = A$. Let $S = \sum_{r \in \mathcal{S}_1} a_r$. Then $\sum_{r \in \mathcal{S}_2} a_r = A - S$, and the objective can be written as $\mathcal{G}(X) = -\left| S - \frac{A}{2} \right| - \left| (A - S) - \frac{A}{2} \right|$. Noting that $\left| (A - S) - \frac{A}{2} \right| = \left| \frac{A}{2} - S \right| = \left| S - \frac{A}{2} \right|$, we get $\mathcal{G}(X) = -2 \left| S - \frac{A}{2} \right|$. Thus, maximizing $\mathcal{G}(X)$ is equivalent to minimizing $\left| S - \frac{A}{2} \right|$. The {optimal objective} value {equals zero} if and only if $S=A/2$, which corresponds exactly to a valid partition. {For} {$N\ge2$}{, a partition with equal positive sums has two nonempty parts, as required by the allocation constraints. The constructed set functions can be evaluated using sums of the input integers}. Therefore, solving the heterogeneous allocation problem allows us to determine whether a solution to the Partition problem exists. Since Partition is NP-complete, the heterogeneous allocation problem is NP-hard. {This construction is a special case of} Equation~\eqref{eq:opt_heterogeneous_centralized}, {so the hardness conclusion applies to that} general {assignment} problem.
\end{proof}

{The hardness result does not rely on monotonicity or diminishing returns} (assumptions made in the homogeneous formulation), and therefore applies to the general heterogeneous formulation where mission evaluation functions are arbitrary set-dependent functions.

\begin{corollary} \label{co:corollary}
    The heterogeneous, iterative, and altruistic allocation problem in Equation~\eqref{eq:opt_heterogeneous_iterative} is NP-hard.
\end{corollary}
\begin{proof}
    The iterative formulation in Equation~\eqref{eq:opt_heterogeneous_iterative} restricts moves to edge-admissible transfers via variables \( \{y^{k+1}_{r,i\to j}\} \), as per Equation~\eqref{eq:binary_transfer}. {To establish hardness under this restriction, use the Partition integers from Proposition~\ref{pr:proposition}, with $\mathcal U=\{1,\ldots,n\}$, $n\ge2$, and $A=\sum_{r\in\mathcal U}a_r$. Take two adjacent teams with unit weights, zero transfer costs, and initial sets $\mathcal S_1^k=\mathcal U\cup\{d\}$ and $\mathcal S_2^k=\{e\}$, where $d,e$ are additional robots. Define}
    {
    \[
        \begin{aligned}
            \mathcal F_1(\mathcal S)&=\mathbb I\{d\in\mathcal S\},\\
            \mathcal F_2(\mathcal S)&=\mathbb I\{e\in\mathcal S\}
            -\left|2\sum_{r\in\mathcal S\cap\mathcal U}a_r-A\right|.
        \end{aligned}
    \]
    Each numbered robot's transfer to team 2 has loss zero and benefit $A-|2a_r-A|>0$. Robots $d,e$ each have transfer benefit zero and loss one, so they remain in place. Thus, every subset $\mathcal I\subseteq\mathcal U$ can transfer while both teams stay nonempty, giving objective $2-|2\sum_{r\in\mathcal I}a_r-A|$. Its optimum equals $2$ if and only if an equal partition exists. This polynomial construction derives all admissible transfers from Hamilton's rule.
    }
    Hence, the optimization problem in Equation~\eqref{eq:opt_heterogeneous_iterative} remains NP-hard in general.
\end{proof}

{Proposition~\ref{pr:proposition} and Corollary~\ref{co:corollary} establish worst-case complexity for the general set-dependent formulation. They do not establish hardness for every restricted mission model or imply that every instance is computationally difficult.}

{Robots may share a capability vector while contributing differently to a mission or incurring different transfer costs. A common capability type therefore does not, by itself, make robots interchangeable. When exchanging robots of the same type leaves mission values, transfer costs, and feasible transfers unchanged, the allocation can instead be represented by the number of robots of each type assigned to each team. This removes equivalent choices, although interactions between types in the mission functions can still make the allocation problem combinatorial.}

{Limiting each team to at most one transfer per type could further reduce the choices considered at each iteration. However, it would exclude simultaneous transfers allowed by Equation~\eqref{eq:opt_heterogeneous_iterative} and could change its optimal solution. The robot-level formulation accommodates differences among robots of the same capability type and allows several of them to transfer simultaneously. Thus, sharing a capability type alone does not recover the homogeneous optimality guarantee, which relies on the count-dependent, diminishing-returns assumptions and exchange rule in~\cite{karam2025resource}.}

\section{Learning Altruistic Collaboration} \label{sec:learning}

{The general heterogeneous assignment problem motivates a scalable learned approximation, while exact one-step optimization supplies labels on tractable instances. We formulate the policy as a structured prediction task over a graph under CTDE. The learned policy maps the current multi-team state to robot-level decisions approximating Equation}~{\eqref{eq:opt_heterogeneous_iterative}.}

\subsection{{One-Step Optimization Scope}} \label{ssec:optimization_scope}

{At each iteration, the policy seeks an allocation that balances the current mission value against the cost of transferring robots. Exact solutions of Equation~\eqref{eq:opt_heterogeneous_iterative} provide optimal labels for this decision. The transfer cost is measured from the current assignment $X^k$, whereas the static problem in Equation~\eqref{eq:opt_heterogeneous_centralized} measures it from $X^0$. Neither objective accounts for the full sequence of future decisions.}

{A sequence of optimal one-step decisions need not form an optimal mission trajectory: moving a robot changes team composition and can affect later rewards, transfer costs, and admissible actions. For a finite-horizon objective that sums rewards over time, one-step optimization is sufficient if, at every stage, the best achievable future reward is the same for all feasible next states. This condition is not assumed here. The homogeneous guarantee in~\cite{karam2025resource} instead concerns the final allocation under identical resources, count-dependent diminishing returns, and a specified exchange rule; it does not directly extend to heterogeneous teams or evolving missions.}

{The learned policy approximates these one-step solutions using local information. Hamilton's rule screens individual transfers, and the execution rule checks local feasibility, but neither ensures that simultaneous transfers maximize or improve the joint objective. The learned policy therefore does not inherit the exact optimizer's guarantee.}

\subsection{Synthetic Data Generation} \label{ssec:generation}

To train the GNN policy, we require labeled data corresponding to one-step collaboration decisions under the heterogeneous altruistic optimization problem in Equation~\eqref{eq:opt_heterogeneous_iterative}. Since exact optimization is intractable at large scale, \textit{small-scale} synthetic problem instances are generated for which the exact one-step solution can be computed via exhaustive search. These exact solutions serve as supervisory labels during training.

Each generated instance defines a supervised learning sample of the form \( (\Gamma, Y) \), where \( \Gamma \) encodes the current system state as a graph and \( Y \) specifies the optimal next assignment for all robots. The learning objective is to approximate the mapping from the current state to the optimal reassignment under the heterogeneous altruistic objective. Each synthetic instance is defined by a team interaction graph, team-specific mission evaluation functions, a heterogeneous robot population, and an initial assignment. The resulting state is encoded as a graph, and the corresponding label is the optimal next-team decision for each robot.

Given \( M \), an undirected and connected graph \( G = (V,E) \) is sampled. Connectivity is ensured by first constructing a spanning tree, after which additional edges are inserted independently with a {specified} probability. Each team \( v \in V \) is assigned a position in a shared workspace, with a minimum separation constraint to ensure spatial distinction. A region \( \mathcal{D}_v \), representing the region of interest, is assigned to each team to represent its operational area. Pairwise distances between team positions are precomputed and stored in a lookup table, which is later used to evaluate inter-team travel distances, i.e., \( d_{r, \, \mathrm{cur}(r) \to v} \).

Next, a heterogeneous population of \( N \) robots is sampled. Each robot is assigned its capability vector {according to a specified probability distribution over capabilities}. An initial assignment \( X^0 \) is then sampled subject to the feasibility constraints.

Given the current state, we compute the Hamilton-admissible mask \( h^{k+1}_{r,i\to j} \), which defines the feasible destination set for each robot. We then enumerate all feasible next assignments that are consistent with this mask and the feasibility constraints. Among these, the assignment \( X^{\star} \) that maximizes $ \mathcal{G}(X^{k+1}) - \lambda \mathcal{C}(X^k, X^{k+1})$ is selected, as defined in Equation~\eqref{eq:opt_heterogeneous_iterative}. The supervisory label is represented as a categorical assignment for each robot, indicating its destination team at the next iteration.

\begin{algorithm}[!b]
    \caption{Synthetic One-Step Data Generation}
    \label{alg:data_generation}
    \begin{algorithmic}[1]
        \STATE \textbf{Input:} number of samples \( N_{\mathrm{data}} \), team and robot ranges, model parameter \( \lambda \)
        \STATE \textbf{Output:} supervised dataset \( D \)
        \STATE Initialize \( D \leftarrow \emptyset \)
        \FOR{\( n = 1, \dots, N_{\mathrm{data}} \)}
            \STATE Sample an instance size \( (M, N) \)
            \STATE Sample a connected team interaction graph \( G = (V, E) \)
            \STATE Sample team positions and corresponding team regions \( \{ \mathcal{D}_v \}_{v \in V} \)
            \STATE Compute the pairwise team-distance lookup table
            \STATE Sample team weights \( \{w_v\}_{v\in V} \)
            \FOR{each team \( v \in V \)}
                \STATE Generate and sample mission-specific team parameters
            \ENDFOR
            \STATE Sample \( N \) heterogeneous robots with corresponding sampled capability vectors
            \FOR{each robot \( r \in \mathcal{R} \)}
                \STATE Generate and sample mission-specific robot parameters
            \ENDFOR
            \STATE Sample an initial assignment \( X^k \)
            \STATE Compute the Hamilton-admissible mask \( h^{k+1}_{r,i\to j}\)
            \STATE Enumerate all feasible next assignments \( X^{k+1} \) consistent with \( h^{k+1}_{r, i \to j} \)
            \STATE Evaluate each feasible assignment using
            \[
                \mathcal{G}(X^{k+1}) - \lambda \mathcal{C}(X^k, X^{k+1})
            \]
            \STATE Let \( X^{\star}\) be the maximizing assignment
            \STATE Construct graph features from \( X^k \) and label the sample with \(X^{\star} \)
            \STATE Append the resulting graph-label pair to \( D \)
        \ENDFOR
        \STATE \textbf{return} \( D \)
    \end{algorithmic}
\end{algorithm}

Each instance is then encoded as a graph for training. Team features include the team weight, number of robots, number of robots under each capability type, mission-specific parameters, and mission value. Robot features include capability indicators and mission-specific parameters. Edge features include relative team position, pairwise distance, the mission-weight ratio \( w_j / w_i \), and the adjacency indicator. In addition, each robot is associated with a Hamilton mask, a candidate destination mask, and a travel-distance matrix. The target label specifies the optimal destination team for each robot.

Algorithm~\ref{alg:data_generation} summarizes the overall data-generation procedure. {In the reported dataset, a} {$60$}{-minute wall-clock budget applies to generating one instance, including its exact label; instances exceeding this budget are replaced. Thus,} {$N_{\mathrm{data}}$} {counts retained labeled instances rather than all attempted instances. The labels supervise Equation}~{\eqref{eq:opt_heterogeneous_iterative} at the sampled current state.}

\subsection{GNN Architecture} \label{ssec:gnn_arch}

In this section, we introduce a policy that approximates the one-step optimal reassignment while remaining computationally efficient at execution time. {We formulate the} problem as masked {multiclass} prediction over a graph and train a GNN to map the heterogeneous state to the next robot-to-team assignment.

At iteration \( k \), the system state is represented by a graph
\[
    \mathcal{H}^k = \big( G, \mathcal{Z}^k, \mathcal{U}^k, \mathcal{E}^k \big),
\]
where \( G \) is the team interaction graph, \( \mathcal{Z}^k = \{ \mathbf{z}_v^k \}_{v \in V} \) denotes team features, \( \mathcal{U}^k = \{ \mathbf{u}_r^k \}_{r \in \mathcal{R}} \) denotes robot features, and \( \mathcal{E}^k = \{ \mathbf{e}_{ij}^k \}_{(i,j) \in V \times V} \) denotes directed edge features. The team features \( \mathbf{z}_v^k \) summarize the current state of team \( v \), including its weight \( w_v \), current team size \( n_v^k \), mission-specific parameters, and any additional quantities required to evaluate \( \mathcal{F}_v(\mathcal{S}_v^k) \). The robot features \( \mathbf{u}_r^k \) encode the capability vector \( \mathbf{\kappa}_r \), robot-specific parameters, and any state variables relevant to transfer decisions. The edge features \( \mathbf{e}_{ij}^k \) encode pairwise information between teams, such as adjacency, relative Euclidean distance, pairwise transfer cost descriptors, and the weight ratio \( w_j / w_i \).

For each robot \( r \), the action space is restricted by the Hamilton-admissible indicator. Specifically, given the current team \( \mathrm{cur}(r) \), the feasible destination set is
\[
    \mathcal{A}_r^k = \{\mathrm{cur}(r)\} \cup
    \left\{
        j \in V \; \middle | \;
        h^{k + 1}_{r, \mathrm{cur}(r) \to j}=1
    \right\},
\]
which includes the option of staying in the current team together with all Hamilton-admissible destination teams. The learning problem is therefore to predict, for each robot \( r \), its next team \( \mathrm{next}(r) \in \mathcal{A}_r^k \).

The proposed GNN consists of three stages:
\begin{enumerate}
    \item feature encoding,
    \item team-level message passing, and
    \item robot-level action scoring.
\end{enumerate}
First, team and robot features are embedded into a common latent space using learnable encoders,
\[
    \mathbf{h}_v^{(0)} = \phi_T(\mathbf{z}_v^k),
    \qquad
    \mathbf{g}_r = \phi_R(\mathbf{u}_r^k),
\]
where \( \phi_T \) and \( \phi_R \) are multilayer perceptrons (MLPs). Next, information is propagated over the team interaction graph through \( L \) rounds of message passing. {The policy uses} {$L=1$} {team-level message-passing layer. The decoder then refines transfer scores locally, without additional communication between teams}. At layer \( \ell \), the message from team \( i \) to team \( j \) is
\[
    \mathbf{m}_{i\to j}^{(\ell)}
    =
    \phi_m \! \left(
        \left[
            \mathbf{h}_i^{(\ell)},
            \mathbf{h}_j^{(\ell)},
            \mathbf{e}_{ij}^k
        \right]
    \right),
\]
where \( \phi_m \) is a learnable message function. Messages are then aggregated at each team \( j \) over its neighbors \( \mathcal{N}_j \),
\[
    \bar{\mathbf{m}}_j^{(\ell)}
    =
    \frac{1}{|\mathcal{N}_j|}
    \sum_{i \in \mathcal{N}_j}
    \mathbf{m}_{i \to j}^{(\ell)},
\]
and the team embedding is updated according to
\[
    \mathbf{h}_j^{(\ell+1)}
    =
    \phi_U\!\left(
        \left[
            \mathbf{h}_j^{(\ell)},
            \bar{\mathbf{m}}_j^{(\ell)}
        \right]
    \right),
\]
where \( \phi_U \) is a learnable update map. After \( L \) layers, each team embedding \( \mathbf{h}_v^{(L)} \) summarizes both local team information and contextual information from neighboring teams.

To evaluate candidate transfers, the policy scores each robot-destination pair \( (r, j) \) with \( j \in \mathcal{A}_r^k \). Let \( i = \mathrm{cur}(r) \); the action score is computed as follows
\[
    a_{r, j}
    =
    \phi_A \! \left(
        \left[
            \mathbf{g}_r, \,
            \mathbf{h}_i^{(L)}, \,
            \mathbf{h}_j^{(L)}, \,
            \mathbf{e}_{ij}^k, \,
            \xi_{r,i \to j}^k
        \right]
    \right),
\]
where \( \phi_A \) is an action-scoring MLP and \( \xi_{r, i \to j}^k \) denotes any robot-dependent transfer descriptor, such as a cost or distance feature.

{The policy augments each team embedding with a learned encoding of the sums and means of its sensing and fire-fighting robot embeddings and the corresponding log-counts. After one team-graph message-passing layer, initial robot-destination scores define masked destination probabilities. A source-local residual decoder pools probability-weighted embeddings and capability counts of the other robots owned by the same source, excluding the scored robot. For each candidate destination, it uses both the predicted allocation there and the predicted allocation remaining at the source. One parallel residual update is added to each initial score, giving the final} {$a_{r,j}$} {used in masked maximization. It uses robot, source/destination team and edge features, travel time and the initial score; it performs no online objective search and exchanges no additional inter-source proposals.}

\begin{figure*}[!t]
    \centering
    \includegraphics[width=2\columnwidth]{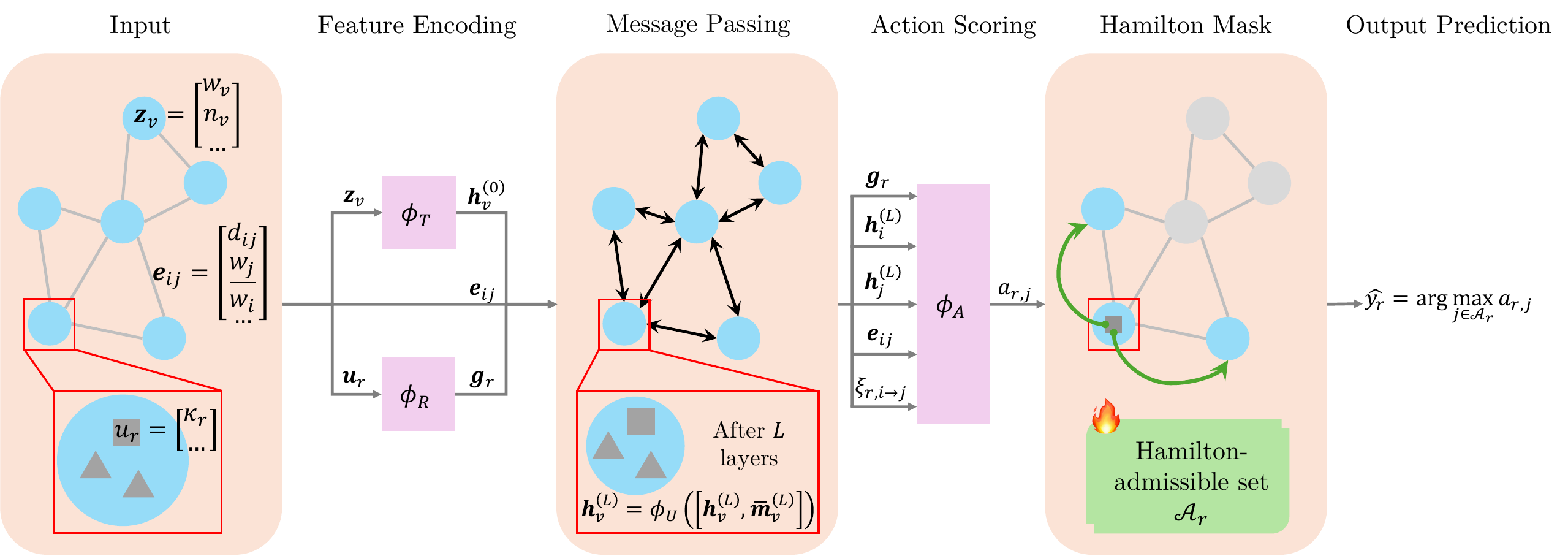}
    \caption{Overview of the proposed GNN-based policy for heterogeneous altruistic collaboration with an example graph input \( G = (V, E) \). The current system state is encoded with node features \( \mathbf{z}_v \), edge features \( \mathbf{e}_{ij} \), and robot features \( \mathbf{u}_r \). Feature encoders map these inputs to latent embeddings, after which message passing over the team interaction graph produces context-aware team representations. For each robot \( r \), candidate transfers \( (r, i \to j) \) are evaluated using an action-scoring function conditioned on the robot embedding, source and destination team embeddings, edge features, and a robot-dependent transfer descriptor \( \xi_{r,i \to j} \). A Hamilton-admissible mask restricts the action set to feasible transfers, and the final decision is obtained via a masked maximization over the admissible set. The {diagram depicts the shared encoding, message-passing and initial action-scoring core. Local robot-set encoding and source-local score refinement complete the policy as described in the text}. The MLPs are represented with the pink squares and rectangle, and the fire emoji highlights the altruistic decision-maker.
    }
    \label{fig:gnn}
\end{figure*}

The predicted next team of robot \(r\) is then given by
\[
    \hat{y}_r^k = \arg\max_{j \in \mathcal{A}_r^k} a_{r, j}.
\]
Equivalently, infeasible actions are masked before the maximization, ensuring that the policy only outputs either Hamilton-admissible decisions or the ``no-move" decision. Training is performed using the exact labels produced by Algorithm~\ref{alg:data_generation} in Section~\ref{ssec:generation}. The overall architecture of the proposed GNN-based policy is illustrated in Figure~\ref{fig:gnn}, which summarizes the encoding of the system state, message passing over the team interaction graph, and the robot-level action scoring with Hamilton-admissible masking. If \( y_r^\star \) denotes the optimal destination team of robot \( r \) in the exact one-step solution \( X^\star \), then the policy parameters are learned by minimizing a masked cross-entropy loss,
\begin{equation} \label{eq:loss_function}
    \mathcal{L}
    =
    - \frac{1}{|\mathcal{R}|}
    \sum_{r \in \mathcal{R}}
    \log \frac{\exp(a_{r, y_r^\star})}{\sum_{j \in \mathcal{A}_r^k} \exp(a_{r,j})}.
\end{equation}
This loss penalizes the model only over feasible actions, thereby aligning the learning objective with the constrained decision structure of Equation~\eqref{eq:opt_heterogeneous_iterative}.

The resulting policy follows a CTDE paradigm, where the model is trained offline using globally generated supervisory labels, while at execution time each team exchanges messages only with its one-hop neighbors, as shown in the next section, and computes robot transfer scores using robot features, source-team and destination-team embeddings, and local edge descriptors. Thus, the learned policy provides a {computationally efficient} approximation of the one-step optimizer {for use within the iterative approach} to Problem~\ref{pl:problem}, {with the scope stated in Section}~{\ref{ssec:optimization_scope}}. Note that the abstract feature definitions above are intentionally mission-agnostic; in any specific application, they are instantiated by choosing team, robot, and edge descriptors consistent with the mission evaluation function \( \mathcal{F}_v(\mathcal{S}_v) \) and transfer cost \( \mathcal{C}(\cdot) \).

\subsection{GNN Inference} \label{ssec:gnn_inference}

Once trained, the GNN is used at inference time as a decentralized decision policy inside an iterative collaboration loop. The goal is to produce a feasible sequence of robot transfers without exhaustively solving Equation~\eqref{eq:opt_heterogeneous_iterative} online. At each decision step, the current system state is encoded as a graph, processed by the learned policy, and then converted into executable transfers through a local feasibility mask.

Given the current assignment \( X^k \), the graph representation \( \mathcal{H}^k \) described in Section~\ref{ssec:gnn_arch} is first constructed. The GNN then performs one round of message passing over the team interaction graph, so that each team exchanges latent information only with its one-hop neighbors. For each robot \( r \in \mathcal{R} \), the policy computes action scores \( a_{r,j} \) over the feasible destination set \( \mathcal{A}_r^k \) and predicts the destination \( \hat{y}_r^k \). The corresponding score \( a_{r,\hat{y}_r^k} \) is used to rank the proposed transfer.

\begin{algorithm}[!b]
    \caption{GNN-Guided Inference}
    \label{alg:gnn_inference}
    \begin{algorithmic}[1]
        \STATE \textbf{Input:} current assignment \( X^k \), trained GNN policy, {iteration cap} {$K_{\max}$}
        \STATE \textbf{Output:} updated assignment \( X^{k+1} \)
        \STATE Initialize \( X^{k+1} \leftarrow X^k \)
        \FOR{{$t=1,\ldots,K_{\max}$}}
            \STATE Construct graph state \( \mathcal{H} \) from \( X^{k+1} \)
            \STATE Compute Hamilton-feasible destination sets \( \{\mathcal{A}_r \}_{r \in \mathcal{R}} \)
            \STATE {Encode local robot sets, perform one team-graph layer, and refine scores with the source-local decoder}
            \FOR{each robot \( r \in \mathcal{R} \)}
                \STATE Predict \( \hat{y}_r = \arg\max_{j \in \mathcal{A}_r} a_{r,j} \)
            \ENDFOR
            \FOR{each team \( i \in V \)}
                \STATE Construct the proposed outgoing transfer set \( \mathcal{P}_i^k \)
                \STATE Sort \( \mathcal{P}_i^k \) by decreasing score
                \FOR{each transfer in \( \mathcal{P}_i^k \)}
                    \IF{executing the transfer preserves local feasibility}
                        \STATE Accept the transfer and update \( X^{k+1} \)
                    \ENDIF
                \ENDFOR
            \ENDFOR
            \IF{no transfer was accepted}
                \STATE \textbf{break}
            \ENDIF
        \ENDFOR
        \STATE \textbf{return} \( X^{k+1} \)
    \end{algorithmic}
\end{algorithm}

Because robot decisions are predicted in parallel, the joint proposal may contain simultaneous transfers that cannot all be executed together while preserving local team feasibility. We therefore interpret the GNN output as a set of candidate transfers and apply a local selection rule. Specifically, for each team \( i \in V \), all proposed outgoing transfers are collected
\[
    \mathcal{P}_i^k =
    \left\{
        \left( r, i \to \hat{y}_r^k, a_{r,\hat{y}_r^k} \right)
        \; \middle| \; \mathrm{cur}(r) = i, \; \hat{y}_r^k \neq i
    \right\},
\]
sorted in descending order of score, and processed sequentially. A proposed transfer is accepted if executing it preserves the hard local constraints of the source team. Otherwise, it is discarded. In the general formulation, these feasibility checks include the assignment constraint and the requirement that each team remain nonempty; additional mission-specific local constraints may also be enforced.

After all teams have processed their proposed outgoing transfers, the accepted moves define the next assignment \( X^{k+1} \). The iterative altruistic collaboration terminates when no team accepts an additional transfer {or when the maximum} {number of iterations $K_{\max}$ is reached}. Algorithm~\ref{alg:gnn_inference} summarizes the inference-time procedure.

\section{Fire-Fighting Scenario} \label{sec:application}

{We instantiate Problem~\ref{pl:problem} in a fire-fighting scenario where teams operate in separate geographical regions and robots can migrate between them. This section defines the mission-specific capabilities, team contributions, and transfer costs used to train and evaluate the GNN.}

For this application, two types of robots are considered: \emph{sensing} robots and \emph{fire-fighting} robots (i.e., water-carrying robots); hence, \( Q = 2 \). Therefore, we get the following capability vector:
\begin{equation} \label{eq:capability}
    \mathbf{\kappa}_r = [S_r \quad W_r] \in \{0, 1\}^2,
\end{equation}
where 
\[
    S_r = 
    \begin{cases}
        1, & \text{if robot \( r \in \mathcal R \) is a \emph{sensing} robot} \\
        0, & \text{otherwise}
    \end{cases}
\]
and
\[
    W_r = 
    \begin{cases}
        1, & \text{if robot \( r \in \mathcal R \) is a \emph{fire-fighting} robot} \\
        0, & \text{otherwise}.
    \end{cases}
\]
Each robot \( r \in \mathcal{R} \) is assigned exactly one of the two capabilities. Team heterogeneity can then be shown through the number of robots under each capability given $N_{W,v}^k = \sum_{r\in \mathcal{S}_v^k} W_r$ and $N_{S,v}^k = \sum_{r \in \mathcal{S}_v^k} S_r$.

In the general iterative heterogeneous collaboration model (Equation~\eqref{eq:opt_heterogeneous_iterative}), we only enforce that each team maintains at least one robot. For this application, we require that this robot be a \textit{sensing} robot; i.e., each team must retain at least one sensing robot at every discrete decision iteration. This application-specific feasibility constraint is enforced in the allocation decision variables (and is not assumed in the mission-agnostic formulation). Thus, the constraint $\sum_{r \in \mathcal{R}} S_r \, x_{r, v} \ge 1, \,\, \forall v \in V$ is added to the optimization problem.

Each robot type contributes differently to the success of the team mission. We define the role of the sensing robots through optimal area coverage, for which a common approach is the Voronoi-based coverage control framework introduced in~\cite{cortes2004coverage}. Specifically, we consider the deployment of \( N_{S,v}^k \in \mathbb{Z}^+ \) sensing robots in team \( v \) over a domain of interest \( \mathcal{D}_v \subset \mathbb{R}^d \), associated with a density function \( \phi_v(q): \mathcal{D}_v \rightarrow \mathbb{R}^+ \) that captures the relative importance of points in \( \mathcal{D}_v \). Assuming that coverage effectiveness decreases as a robot becomes farther from a point \( q \in \mathcal{D}_v \), the locational cost function for team \( v \in V \) can be written as~\cite{cortes2004coverage}
\begin{equation}
    L_v(p) = \sum_{i \in {\mathcal I_v}} \int_{\mathcal{V}_i}\! \|p_i - q\|^2 \phi_v(q) \,dq,
    \label{eq:lcost}
\end{equation}
where \( p_i \in \mathbb{R}^d \) denotes the position of robot \( i \), and the vector \( p = [p_1^\top, p_2^\top,\dots, p_{N_{S,v}^k}^\top]^\top \) contains the positions of all sensing robots in team \( v \). The index set \( {\mathcal I_v} = \{1, \dots, N_{S,v}^k\} \) corresponds to the sensing robots after local relabeling within the team, and
\begin{align} \label{voronoi}
    \mathcal{V}_i &= \left\{ q \in \mathcal{D}_{{v}}
    \,\middle|\,
    \lVert q - p_i \rVert \leq \lVert q - p_j \rVert,\;\forall j \neq i \in {\mathcal I_v}
    \right\}
\end{align}
defines the Voronoi cell associated with robot \( i \in {\mathcal I_v} \). {For distinct sensing positions,} Equation~\eqref{voronoi} {gives} $\bigcup_{i\in{\mathcal I_v}}\mathcal V_i=\mathcal D_{{v}}$ {and pairwise boundary overlaps of zero} Lebesgue measure: $\lambda({\mathcal V_i\cap\mathcal V_j})=0$ {for} {$i\ne j$}. A standard approach for obtaining an optimal coverage configuration, e.g.,~\cite{cortes2004coverage}, is to let each robot \( i \) move according to the negative gradient flow of \( L_v(p) \) with respect to \( p_i \), \( \forall i \in {\mathcal I_v} \). This corresponds to a continuous-time version of Lloyd’s algorithm and asymptotically converges to a centroidal Voronoi tessellation (CVT), given by
\begin{equation} \nonumber
    p_i = c_i = \frac{\int_{\mathcal{V}_i} \! q \phi_v(q) \,dq}{\int_{\mathcal{V}_i} \! \phi_v(q) \,dq}, \quad \forall i \in {\mathcal I_v},
\end{equation}
where \( c_i \in \mathbb{R}^d \) denotes the centroid of the Voronoi cell of robot \( i \), and the resulting CVT configuration is represented as \( c = [c_1^\top, c_2^\top,\dots, c_{N_{S,v}^k}^\top]^\top \). Finally, we explicitly include the number of sensing robots \( N_{S,v}^k \) in team \( v \) as an argument of the locational cost function in Equation~\eqref{eq:lcost}, yielding
\begin{equation} \label{eq:lcost_cdc}
    L_v(N_{S,v}^k, p) = \sum_{i=1}^{N_{S,v}^k} \int_{\mathcal{V}_i}\! \|p_i - q\|^2 \phi_v^k(q)\,dq.
\end{equation}

{Since the fire density is piecewise constant, the density of team \(v\) at iteration \(k\), denoted by \(\phi_v^k(q)\), is included as an argument of the locational cost function in Equation~\eqref{eq:lcost_cdc}, giving}
\begin{equation} \label{eq:lcost_new}
    L_v(N_{S,v}^k, p, \phi_v^k(q)) = \sum_{i=1}^{N_{S,v}^k} \int_{\mathcal{V}_i}\! \|p_i - q\|^2 \phi_v^k(q) \, dq.
\end{equation}
{To include the case with no sensing robots, we define the locational cost as}
\begin{equation} \label{eq:lcost_new_new}
    L_v = 
    \begin{cases}
        0, & \text{if} \,\, N_{S,v}^k = 0, \\
        \sum_{i=1}^{N_{S,v}^k} \int_{\mathcal{V}_i}\! \|p_i - q\|^2 \phi_v^k(q) \, dq, & \text{otherwise},
    \end{cases}
\end{equation}
since the locational cost in Equation~\eqref{eq:lcost_new} is not defined at \( N_{S,v}^k = 0 \).

{We use a sigmoid function of the inverse locational cost in Equation~\eqref{eq:lcost_new_new} to obtain a bounded sensing contribution that increases as coverage improves.} The general form of the sigmoid function is defined as
\begin{equation} \label{eq:sig}
    \sigma (x) = \frac{1}{1+e^{- a (x - b)}}
\end{equation}
where \( b \) is the value of $x$ at which \( \sigma (x) = 0.5 \).
{The resulting sensing contribution for team $v$ is}
\begin{equation} \label{eq:sensing_func}
    \psi_v(N_{S,v}^k) =
    \begin{cases}
        0, & \text{if} \,\, N_{S,v}^k = 0, \\
        \frac{1}{1+e^{-a(1/{L_v} - b)}}, & \text{otherwise}
    \end{cases}
\end{equation}
{With $a=1$ and $b=0$, the sensing contribution in Equation~\eqref{eq:sensing_func} lies in $[0.5,1]$ for finite $L_v>0$ when sensing robots are present. It approaches $1$ as $L_v\to0^+$ and $1/2$ as $L_v\to\infty$. Without sensing robots, it is separately set to zero; thus, values in these cases lie in $\{0\}\cup[0.5,1]$. Once sensing is available, improving coverage can increase this multiplier by less than a factor of two. The sensing-feasibility constraint ensures that every team has at least one sensing robot in each executed allocation.}

{Fire-fighting robots contribute through their water-carrying and discharge capacities. Let $K_r$ denote the capacity of robot $r$; the team's aggregate fire-fighting power is}
\begin{equation} \label{eq:fire_fight_func}
    P_v(N_{W,v}^k) = \sum_{r \in \mathcal{S}_v^k : W_r = 1} K_r.
\end{equation}

Since the fire density is piecewise constant as mentioned previously, we define the mission evaluation function for each team \( v \in V \) as:
\begin{equation} \label{eq:mission_eval_func_formula_k+1}
    \mathcal{F}_v(\mathcal{S}_v^{k+1}) = - \int_{\mathcal{D}_v} \phi_v^{k+1}(q) \, dq,
\end{equation}
with
\begin{equation} \label{eq:phi_update_rule}
    \phi_v^{k+1}(q) = \phi_v^k(q) \, e^{-\frac{1}{\eta} \, P_v^{k+1}(N_{W,v}^{k+1}) \, \psi_v^{k+1}(N_{S,v}^{k+1}) \, \Delta t},
\end{equation}
where \( \eta \in \mathbb{R}^+ \) represents the speed of the exponential decay and \( \Delta t > 0 \) is the time step between iterations. This formulation captures collaborative capability complementarity (i.e., \cite{karam2026collaboration}) as sensing robots improve situational awareness through coverage, while fire-fighting robots reduce fire intensity. Neither capability alone is sufficient to suppress fires, and effective performance emerges from their joint contribution.

\begin{table}[!t]
    \centering
    \renewcommand{\arraystretch}{1.5}
    \caption{Fire-Fighting Model Components}
    \label{tab:fire_model}
    \begin{tabular}{c|l}
        \hline
        \textbf{Component} & \textbf{Description} \\
        \hline
        \( S_r \) & Sensing capability indicator of robot \( r \) \\
        \hline
        \( W_r \) & Fire-fighting capability indicator of robot \( r \) \\
        \hline
        \( N_{S,v}^k \) & Number of sensing robots in team \( v \) at iteration \( k \) \\
        \hline
        \( N_{W,v}^k \) & Number of fire-fighting robots in team \( v \) at iteration \( k \) \\
        \hline
        \( \phi_v^k(q) \) & Fire density at location \( q \) for team \( v \) at iteration \( k \) \\
        \hline
        \( \mathcal{D}_v \) & Spatial domain associated with team \( v \) \\
        \hline
        \( L_v \) & Locational cost for sensing robots (coverage quality) \\
        \hline
        \( \psi_v(N_{S,v}^k) \) & Sensing effectiveness (sigmoid of inverse locational cost) \\
        \hline
        \( K_r \) & Fire-fighting capacity of robot \( r \) \\
        \hline
        \( P_v(N_{W,v}^k) \) & Aggregate fire-fighting power of team \( v \) \\
        \hline
        \( \tau_{r, \mathrm{cur}(r)\to v} \) & Travel time from current team to team \( v \) \\
        \hline
        \( \alpha \) & Cost scaling parameter \\
        \hline
        \( \eta \) & Fire decay rate parameter \\
        \hline
        \( \Delta t \) & Discrete time step \\
        \hline
    \end{tabular}
\end{table}

{We model the transfer cost as}
\begin{equation} \label{eq:transfer_cost_formula}
    \mathcal{C}(X^k, X^{k+1}) = \sum_{r \in \mathcal{R}} \sum_{v \in V} \alpha \, x^{k+1}_{r,v} \, \tau_{r, \mathrm{cur}(r) \to v},
\end{equation}
while instantiating the generic cost function in Equation~\eqref{eq:transfer_cost}, for this application, as
\begin{equation} \nonumber
    c_{r, v}(x^k_{r, v}) = \alpha \, \tau_{r, \mathrm{cur}(r) \to v},
\end{equation}
and with
\begin{equation} \label{eq:travel_time}
    \tau_{r, \mathrm{cur}(r) \to v} = \frac{d_{r, \mathrm{cur}(r) \to v}}{s_{r}}
\end{equation}
representing the time needed for robot \( r \) {to travel the distance} between its current team and team \( v \), given its speed \( s_r \), and \( \alpha \) being a regularization constant with \( \alpha > 0 \). {Robots with the same capability vector may nevertheless differ in fire-fighting capacity or travel speed, leading to different mission contributions or transfer costs.}

Table~\ref{tab:fire_model} summarizes the key components of the fire-fighting application of the theoretical heterogeneous altruistic collaborative model presented in Section~\ref{ssec:heterogeneous}.

\section{Results} \label{sec:results}

{We evaluate the proposed GNN in the fire-fighting setting, beginning with its ability to approximate exact one-step assignments and the contribution of its allocation decoder. We then examine decision computation and multi-step mission outcomes at larger scales, followed by paired physical trials comparing learned reallocation with fixed team membership.}

\subsection{Machine Learning Framework} \label{ssec:gnn_perf}

\begin{figure}[!b]
    \centering
    \includegraphics[width=\columnwidth]{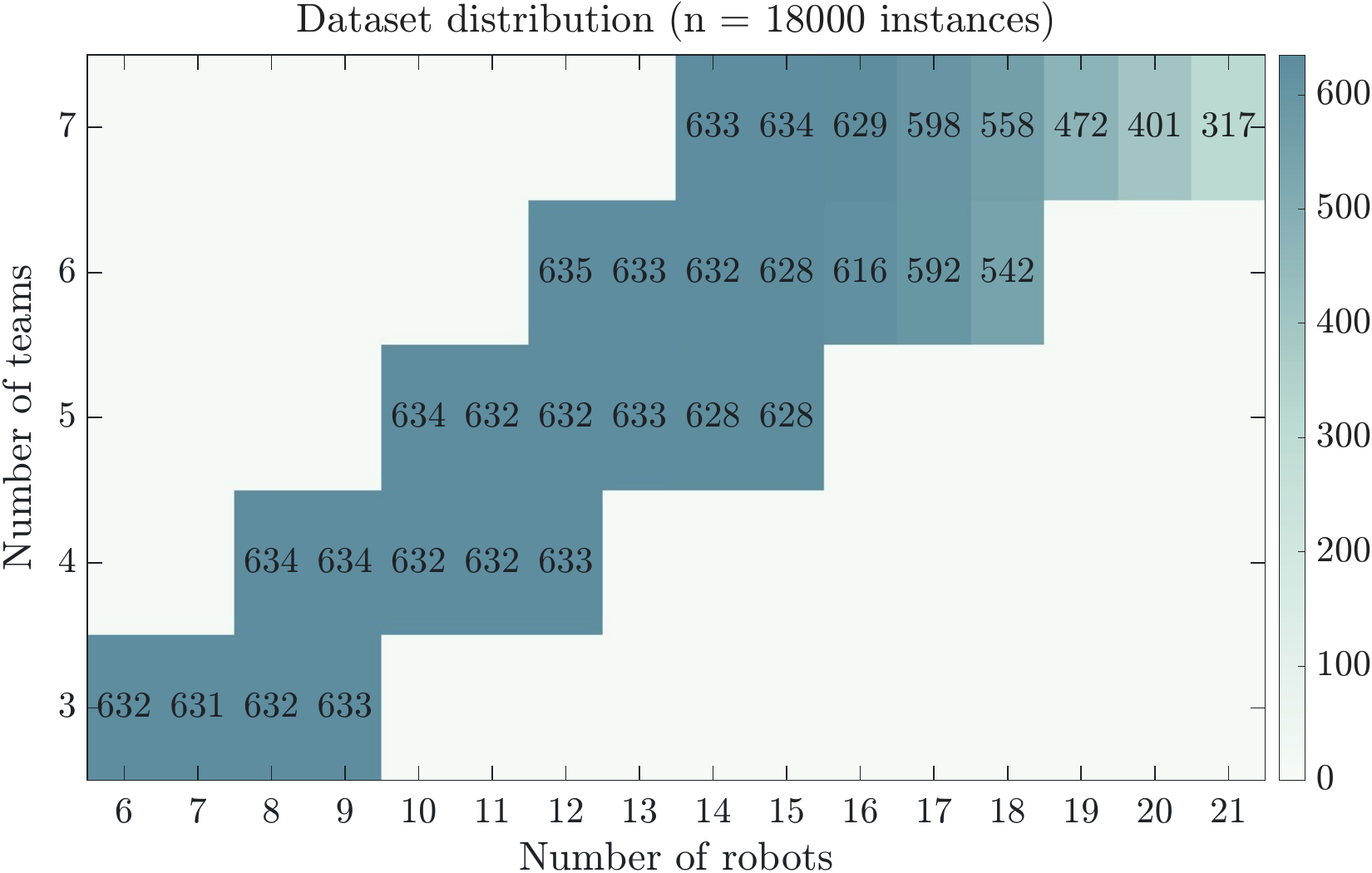}
    \caption{Distribution of the \( 18{,}000 \)-instance dataset over the number
    of teams and robots. The dataset covers heterogeneous collaboration
    problems with \( 3 \) to \( 7 \) teams and robot populations scaled with the
    number of teams. {Cell values count retained labeled instances. With a $60$-minute generation limit per instance, no timeouts were logged for 3- or 4-team instances, $1$ for 5 teams, $91$ for 6 teams, and $600$ for 7 teams ($692$ total logged timeouts).}}
    \label{fig:data-dist}
\end{figure}

The supervised learning framework used to train the proposed GNN policy is evaluated on a dataset of \( 18{,}000 \) synthetically generated heterogeneous collaboration instances. Each instance corresponds to a graph-structured multi-team system with \( 3 \) to \( 7 \) teams and a robot population that scales proportionally with the number of teams. Ground-truth labels are obtained using the exact one-step optimizer described in Algorithm~\ref{alg:data_generation}, providing optimal robot-to-team assignment decisions under the heterogeneous altruistic objective.

{The dataset comprises $18$ batches of $1{,}000$ labeled instances, each generated under a $60$-minute budget per instance that includes exact labeling. Figure~\ref{fig:data-dist} shows their distribution by team and robot count and summarizes the logged timeouts. A fixed random split assigns $14{,}400$ instances to training, $1{,}800$ to validation, and $1{,}800$ to testing throughout the evaluation. Because inclusion in the dataset requires an exact label within the generation budget, the dataset can favor instances that are easier to solve. The supervised results are therefore interpreted within this population, complemented by a separate evaluation on newly generated states at larger team sizes.}

In the generated dataset, the robot-level decision labels are imbalanced, as across the full dataset, \( 192{,}628 \) decisions correspond to staying with the current team, while \( 41{,}381 \) correspond to moving to another team. Thus, only \( 17.68\% \) of robot-level decisions are transfer actions. This imbalance {motivates separate reporting of transfer detection and destination prediction. A high aggregate destination accuracy can be dominated by stay decisions; move-target accuracy evaluates the subset for which the reference assignment calls for a transfer}.

The learning problem is formulated as a masked multi-class classification task. For each robot, the model predicts its next team from a feasible set defined by the Hamilton-admissible mask, including the current team (``stay") and admissible transfer destinations. Since the number of teams varies across instances, the number of candidate classes is variable, with a maximum of seven classes in the generated dataset. The Hamilton mask is applied prior to the softmax operation, ensuring that probability mass is assigned only to feasible actions.

The GNN operates using one-hop message passing over the team interaction graph, consistent with the decentralized execution setting. {Team and robot embeddings have dimension $128$, and the complete model has $279{,}090$ trainable parameters.} The model is {initially} trained for \( 50 \) epochs using the AdamW optimizer~\cite{loshchilov2017decoupled} with a learning rate of \( 10^{-3} \), weight decay of \( 10^{-4} \), dropout rate of \( 0.1 \), and batch size of \( 128 \). All input features are normalized using statistics computed from the training set. The training objective is the masked cross-entropy loss defined in Equation~\eqref{eq:loss_function}, augmented with a binary auxiliary loss that predicts whether a robot should move or stay. This auxiliary loss is weighted by \( 0.15 \), and move-target decisions are further emphasized in the cross-entropy objective by a factor of \( 1.25 \), reflecting the increased difficulty of predicting destination teams among transfer actions. {The residual decoder is then initialized with zero output, and all parameters are jointly refined for $20$ epochs at learning rate $10^{-4}$. The complete procedure is repeated with three random initializations, selecting each final-stage checkpoint by the lowest validation objective gap after the execution filter, evaluated every five epochs.}

{We evaluate destination prediction and transfer detection separately. Let $\mathcal I$ index all evaluated robot decisions. For decision $i$, $c_i$ is the current team, $y_i$ the exact reference destination, $\hat y_i$ the predicted destination, and $\mathcal A_i$ the admissible destination set. The indicators $m_i=\mathbb I\{y_i\ne c_i\}$ and $\hat m_i=\mathbb I\{\hat y_i\ne c_i\}$ identify reference and predicted transfers. Let $\mathcal T_i$ contain the $\min(3,|\mathcal A_i|)$ highest-scoring admissible destinations. We use four complementary accuracy measures:}
{
\begin{equation} \label{eq:evaluation_metrics}
    \begin{aligned}
         A_{\mathrm{exact}}&=\frac{\sum_{i\in\mathcal I}\mathbb I\{\hat y_i=y_i\}}{|\mathcal I|},\\
         A_{\mathrm{move/stay}}&=\frac{\sum_{i\in\mathcal I}\mathbb I\{\hat m_i=m_i\}}{|\mathcal I|},\\
         A_{\mathrm{top3}}&=\frac{\sum_{i\in\mathcal I}\mathbb I\{y_i\in\mathcal T_i\}}{|\mathcal I|},\\
         A_{\mathrm{move\ target}}&=\frac{\sum_{i\in\mathcal I}m_i\mathbb I\{\hat y_i=y_i\}}{\sum_{i\in\mathcal I}m_i}.
    \end{aligned}
\end{equation}
}
{Exact accuracy measures agreement with the reference destination, whereas move/stay accuracy measures agreement only on whether to transfer. Top-3 accuracy allows the reference destination to appear among the three highest-scoring candidates, while move-target accuracy evaluates exact destination agreement only on reference transfers, including those incorrectly predicted as stay. Larger values indicate better agreement within each metric's stated population.}

\begin{table}[!b]
    \centering
    \renewcommand{\arraystretch}{1.5}
    \caption{Multi-Seed Robustness}
    \label{tab:seed_robustness}
    \begin{tabular}{l|c|c|c|c}
        \hline
        \textbf{Metric} & \textbf{Mean} & \textbf{Std.} & \textbf{Min} & \textbf{Max} \\
        \hline
        Overall accuracy & {90.37\%} & {0.05\%} & {90.33\%} & {90.43\%} \\
        Move/stay accuracy & {92.69\%} & {0.04\%} & {92.64\%} & {92.72\%} \\
        Top-3 accuracy & {99.57\%} & {0.02\%} & {99.54\%} & {99.59\%} \\
        Move target accuracy & {67.20\%} & {0.30\%} & {66.86\%} & {67.44\%} \\
        Mean loss & {0.2636} & {0.0017} & {0.2617} & {0.2646} \\
        \hline
    \end{tabular}
\end{table}

{To distinguish missed transfers from unnecessary ones, we also measure transfer precision and recall. Treating a transfer as the positive class gives}
{
\begin{equation} \label{eq:move_detection_metrics}
    \begin{aligned}
         N_{\mathrm{TP}}&=\sum_{i\in\mathcal I}m_i\hat m_i,\\
         N_{\mathrm{FP}}&=\sum_{i\in\mathcal I}(1-m_i)\hat m_i,\\
         N_{\mathrm{FN}}&=\sum_{i\in\mathcal I}m_i(1-\hat m_i),\\
         P_{\mathrm{move}}&=\frac{N_{\mathrm{TP}}}{N_{\mathrm{TP}}+N_{\mathrm{FP}}},\qquad
         R_{\mathrm{move}}=\frac{N_{\mathrm{TP}}}{N_{\mathrm{TP}}+N_{\mathrm{FN}}}.
    \end{aligned}
\end{equation}
}
{Here, $N_{\mathrm{TP}}$ counts correctly detected transfers, $N_{\mathrm{FP}}$ counts transfers predicted when the reference stays, and $N_{\mathrm{FN}}$ counts missed reference transfers. Precision measures how often a predicted transfer is supported by the reference, while recall measures the fraction of reference transfers detected. Because these measures concern the move decision, a wrong destination can count as a correct detection while remaining incorrect for move-target accuracy. All metrics pool robot decisions across each split using masked scores before the execution filter; conditional ratios require nonzero denominators.}

\begin{figure}[!t]
    \centering
    \includegraphics[width=\columnwidth]{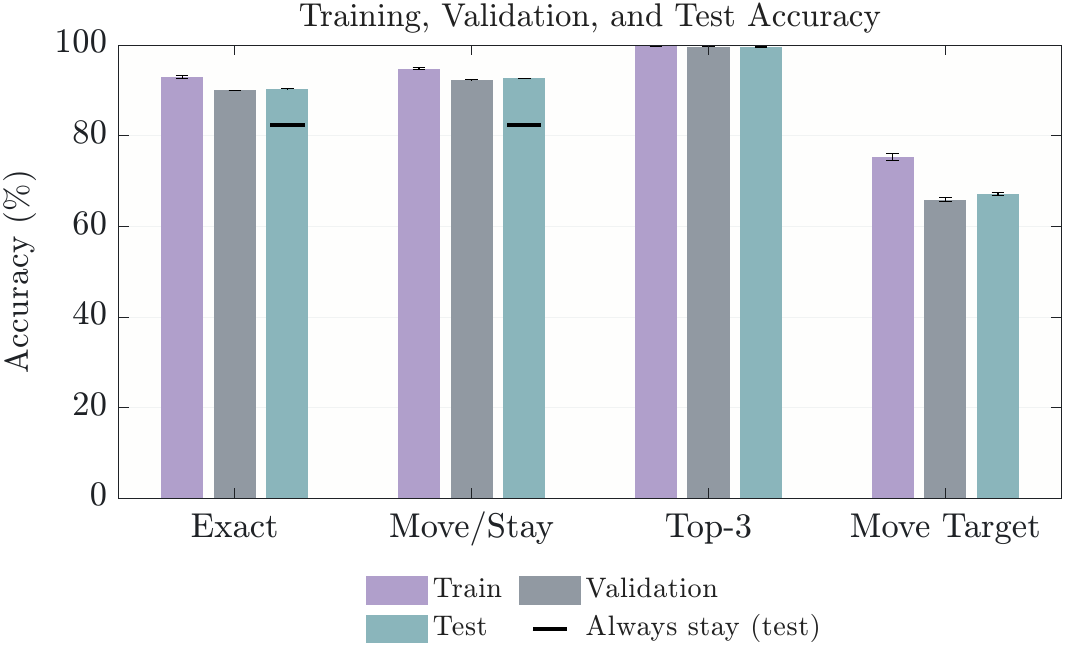}
    \caption{Training, validation, and testing performance of the GNN policy. {Bars show means over three independently initialized training runs; error bars show sample standard deviations. Black markers show the always-stay test baseline for exact and move/stay accuracy.}}
    \label{fig:ml-performance}
\end{figure}

{Label agreement alone does not establish allocation quality, particularly because top-3 membership is automatic when $|\mathcal A_i|\le3$. We therefore also evaluate the objective attained after the execution filter and the mission outcomes produced by repeated allocation decisions.}

{Table~\ref{tab:seed_robustness} and Figure~\ref{fig:ml-performance} summarize the results from three training runs, with test accuracies and sample standard deviations reported in percent and loss reported as a dimensionless value. Mean test exact accuracy is $90.37\%$, with a sample standard deviation of $0.05$ percentage points across initializations; the always-stay baseline achieves $82.31\%$. Transfer-conditioned evaluation provides a more specific view: move-target accuracy is $67.20\%$, while move precision and recall are $78.78\%$ and $80.30\%$, respectively. These measures distinguish reliable detection of a transfer opportunity from the more demanding choice of its destination.}

{Figure~\ref{fig:training-dynamics} shows training and validation accuracy and loss during the $50$-epoch initial stage, before decoder refinement. The plotted loss combines move-weighted cross-entropy, normalized by the total decision weight, with the mean binary auxiliary loss weighted by $0.15$. Disabling dropout for both splits makes their learning curves directly comparable.}

\subsection{{One-Step Allocation Evaluation}} \label{ssec:additional_evaluation}

{We first assess how closely the executed allocation approaches the exact one-step solution. Let $J^\star$, $J_{\mathrm{method}}$, and $J_{\mathrm{stay}}$ denote the objective values in Equation~\eqref{eq:opt_heterogeneous_iterative} for the exact reference allocation, the evaluated method, and the unchanged allocation. For $|J_{\mathrm{stay}}|>0$, the normalized objective gap is}
{
\begin{equation}\label{eq:objective_gap}
    \mathrm{Gap}=100\frac{J^\star-J_{\mathrm{method}}}{|J_{\mathrm{stay}}|}.
\end{equation}
}
{The normalized gap measures the objective shortfall relative to the magnitude of the unchanged allocation's objective: smaller values indicate better allocations, and zero matches the exact reference. The proposed model and its ablation use the same mission objective, Hamilton mask, transfer costs, and sensing constraints.}

{To assess the benefit of considering other robots’ predicted destinations, we compare the proposed decoder with a uniform-context ablation that replaces predicted destination probabilities with uniform probabilities over admissible destinations. Both models share the same split, parameter count, initialization, normalization, data order, optimizer, exact labels, and refinement budget. Results are averaged over three training runs with validation-based checkpoint selection. Because the test data were also examined during model development, we additionally evaluate the GNN on newly generated states.}

{On $1{,}800$ test states, predicted-allocation context reduces the mean objective gap from $9.64\%$ to $6.85\%$, a relative reduction of $29.0\%$. A corresponding decrease from $10.13\%$ to $7.26\%$ on validation supports the benefit of accounting for other robots' predicted destinations when scoring a transfer.}

\begin{figure}[!t]
    \centering
    \includegraphics[width=\columnwidth]{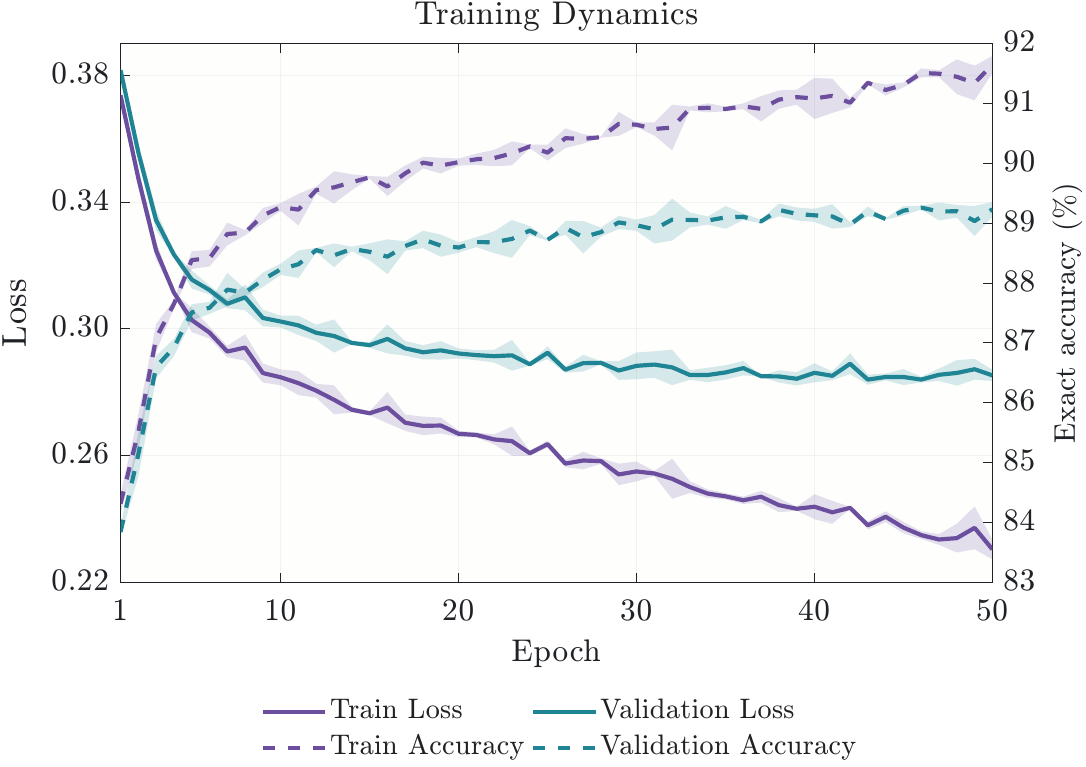}
    \caption{Training dynamics of the GNN policy, showing the model loss and exact accuracy for both the training and validation splits. {The plot covers the $50$-epoch initial training stage. Solid lines show supervised loss (left axis); dashed lines show exact accuracy (right axis). Purple denotes training and teal denotes validation. Metrics are evaluated on the complete splits with dropout disabled after each completed epoch. Curves show means and shaded bands show sample standard deviations across 3 runs.}}
    \label{fig:training-dynamics}
\end{figure}

\subsection{{Computation and Allocation Quality}} \label{ssec:quality_time}

\begin{table*}[!t]
    \centering
    \renewcommand{\arraystretch}{1.2}
    \caption{Runtime scaling of the exact optimization and the GNN policy. {Total runtime covers the complete allocation loop for each method; mean step time is the total divided by the number of decisions.} {Dashes indicate a timeout at the 60-minute limit.}}
    \label{tab:runtime-scaling}
    \begin{tabular}{c|c|c|c|c|c|c|c}
        \hline
        \textbf{Teams} & \textbf{Robots} & \textbf{Opt. Total (s)} & \textbf{Opt. Mean Step (s)} & \textbf{Opt. Steps} & \textbf{GNN Total (s)} & \textbf{GNN Mean Step (s)} & \textbf{GNN Steps}\\
        \hline
        3 & 9 & 0.051 & 0.025 & 2 & {0.134} & {0.027} & {5}\\
        4 & 12 & 0.090 & 0.030 & 3 & {0.064} & {0.021} & {3}\\
        5 & 15 & 0.781 & 0.390 & 2 & {0.098} & {0.049} & 2\\
        6 & 18 & 35.4 & 11.8 & 3 & {0.228} & {0.046} & 5\\
        7 & 21 & {---} & {---} & 1 & {0.292} & {0.058} & {5}\\
        8 & 24 & {---} & {---} & {---} & {0.168} & {0.056} & {3}\\
        9 & 27 & {---} & {---} & {---} & {0.269} & {0.134} & {2}\\
        10 & 30 & {---} & {---} & {---} & {0.581} & {0.145} & {4}\\
        15 & 45 & {---} & {---} & {---} & {1.046} & {0.209} & {5}\\
        25 & 75 & {---} & {---} & {---} & {5.844} & {0.730} & {8}\\
        35 & 105 & {---} & {---} & {---} & {17.139} & {1.318} & {13}\\
        40 & 120 & {---} & {---} & {---} & {16.112} & {1.790} & {9}\\
        45 & 135 & {---} & {---} & {---} & {51.333} & {2.232} & {23}\\
        50 & 150 & {---} & {---} & {---} & {66.074} & {2.753} & {24}\\
        \hline
    \end{tabular}
\end{table*}

{Table~\ref{tab:runtime-scaling} compares the computational requirements of exhaustive optimization and the learned GNN policy on the same benchmark instances and initial assignments. Each row represents one run per method and system size. Total runtime covers the complete allocation loop, including feature construction, Hamilton-admissibility checks, network inference, and local filtering for the GNN. Mean step time is the total runtime divided by the number of allocation decisions, distinguishing the average cost of a decision from the length of the allocation loop. A $60$-minute timeout is imposed on exact optimization.}

{Both methods complete the three- to five-team examples in less than one second. At six teams and eighteen robots, exhaustive optimization requires $35.4$~s over three decisions, whereas the GNN requires $0.228$~s over five decisions. For seven teams and twenty-one robots, exhaustive optimization becomes computationally intractable and reaches the timeout during its first decision, while the GNN executes five decisions in $0.292$~s.} As the system size increases further, the runtime advantage of the learned policy becomes increasingly pronounced. The learned policy remains computationally tractable in regimes where exact optimization is no longer practically executable. {For fifty teams and $150$ robots, the GNN executes $24$ decisions in $66.074$~s, averaging $2.753$~s per decision. Decision time increases with system size, but the policy remains executable well beyond the training range.}

\begin{table*}[!b]
    \centering
    \renewcommand{\arraystretch}{1.2}
    \caption{{One-decision comparison on the Table~\ref{tab:runtime-scaling} instances. Search stops at its first objective improvement or after $600$ s. GNN times are medians of three repetitions; search times are single measurements. Improvements are relative to staying, with larger values indicating better allocations. Objective improvement includes transfer costs, whereas mission improvement excludes them. $\dagger$~denotes a timeout with the unchanged allocation as the feasible fallback and zero improvement.}}
    \label{tab:first_improvement}
    {
        \newlength{\decisioncolwidth}
        \newlength{\objectivecolwidth}
        \newlength{\missioncolwidth}
        \settowidth{\decisioncolwidth}{\textbf{Decision time (s)}}
        \settowidth{\objectivecolwidth}{\textbf{Objective improvement (\%)}}
        \settowidth{\missioncolwidth}{\textbf{Mission improvement (\%)}}
        \setlength{\decisioncolwidth}{\dimexpr\decisioncolwidth/2-\tabcolsep\relax}
        \setlength{\objectivecolwidth}{\dimexpr\objectivecolwidth/2-\tabcolsep\relax}
        \setlength{\missioncolwidth}{\dimexpr\missioncolwidth/2-\tabcolsep\relax}
        \begin{tabular}{c|c|*{2}{>{\centering\arraybackslash}p{\decisioncolwidth}|}*{2}{>{\centering\arraybackslash}p{\objectivecolwidth}|}>{\centering\arraybackslash}p{\missioncolwidth}|>{\centering\arraybackslash}p{\missioncolwidth}}
            \hline
            & & \multicolumn{2}{c|}{\textbf{Decision time (s)}} & \multicolumn{2}{c|}{\textbf{Objective improvement (\%)}} & \multicolumn{2}{c}{\textbf{Mission improvement (\%)}}\\
            \hline
            \textbf{Teams} & \textbf{Robots} & \textbf{GNN} & \textbf{Search} & \textbf{GNN} & \textbf{Search} & \textbf{GNN} & \textbf{Search}\\
            \hline
            3 & 9 & 0.014 & 0.017 & 44.19 & 56.93 & 45.91 & 59.14\\
            4 & 12 & 0.021 & 0.021 & 4.37 & 10.67 & 7.73 & 12.40\\
            5 & 15 & 0.049 & 0.048 & 14.35 & 14.35 & 16.67 & 16.67\\
            6 & 18 & 0.045 & 0.053 & 38.87 & 23.12 & 39.56 & 24.53\\
            7 & 21 & 0.063 & 6.163 & 5.49 & 0.11 & 6.08 & 3.58\\
            8 & 24 & 0.058 & 0.072 & 35.49 & 4.36 & 36.85 & 6.85\\
            9 & 27 & 0.148 & 11.862 & 51.37 & 6.69 & 53.19 & 9.65\\
            10 & 30 & 0.161 & 0.191 & 39.04 & 6.01 & 39.73 & 8.00\\
            15 & 45 & 0.235 & $600^\dagger$ & 36.81 & 0.00 & 37.81 & 0.00\\
            25 & 75 & 0.762 & $600^\dagger$ & 15.49 & 0.00 & 17.12 & 0.00\\
            35 & 105 & 1.393 & $600^\dagger$ & 22.69 & 0.00 & 24.38 & 0.00\\
            40 & 120 & 2.000 & $600^\dagger$ & 18.24 & 0.00 & 19.15 & 0.00\\
            45 & 135 & 2.323 & $600^\dagger$ & 29.92 & 0.00 & 31.65 & 0.00\\
            50 & 150 & 2.962 & $600^\dagger$ & 15.89 & 0.00 & 16.84 & 0.00\\
            \hline
        \end{tabular}
    }
\end{table*}

{To examine allocation quality alongside computation, we evaluate the GNN on $60$ newly generated states spanning $3,5,7,10,15,$ and $20$ teams, with either two or three robots per team on average ($6$-$60$ robots). Five states are generated for each size combination using the same mission parameters and density model. Without exact labels for these states, we measure the improvement over the unchanged allocation:}
{
\begin{equation}\label{eq:improvement_over_stay}
    \mathrm{Improvement}=100\frac{J_{\mathrm{method}}-J_{\mathrm{stay}}}{|J_{\mathrm{stay}}|}.
\end{equation}
}
{A positive value indicates improvement over staying, while zero leaves the objective unchanged; this measure does not quantify distance from an unknown optimum. Quality is averaged across three trained GNNs within each state, and decision times are measured over three repetitions in the same computing environment.}

{At twenty teams, the GNN achieves a mean objective improvement of $24.47\%$ with a median complete decision time of $735.32$ ms. Neural forward evaluation alone takes $0.923$ ms, while feature and Hamilton-mask construction account for $99.87\%$ of aggregate GNN decision time. Thus, the cost of constructing the mission-specific policy inputs dominates end-to-end computation, despite the network's single message-passing layer.}

{To assess how quickly a beneficial reassignment becomes available, Table~\ref{tab:first_improvement} compares one GNN decision with exhaustive search stopped at its first improvement on the same $14$ benchmark instances and initial assignments as Table~\ref{tab:runtime-scaling}. The search follows the existing enumeration order and stops when $J>J_{\mathrm{stay}}+10^{-9}$ or a $10$-minute decision budget is exhausted. Since the unchanged allocation is already feasible, a timeout returns that allocation with zero improvement. Objective improvement is measured by Equation~\eqref{eq:improvement_over_stay}; mission improvement uses the same normalization with $\mathcal G$ in place of $J$, separating the suppression benefit from the transfer-cost penalty. GNN times are medians of three repetitions, while each search is run once with the $600$-second cap. Five additional paired repetitions for the three smallest cases confirm essentially equal decision times at four and five teams.}

\begin{figure*}[!t]
    \centering
    \subfloat[Initial Allocation]{\resultframe{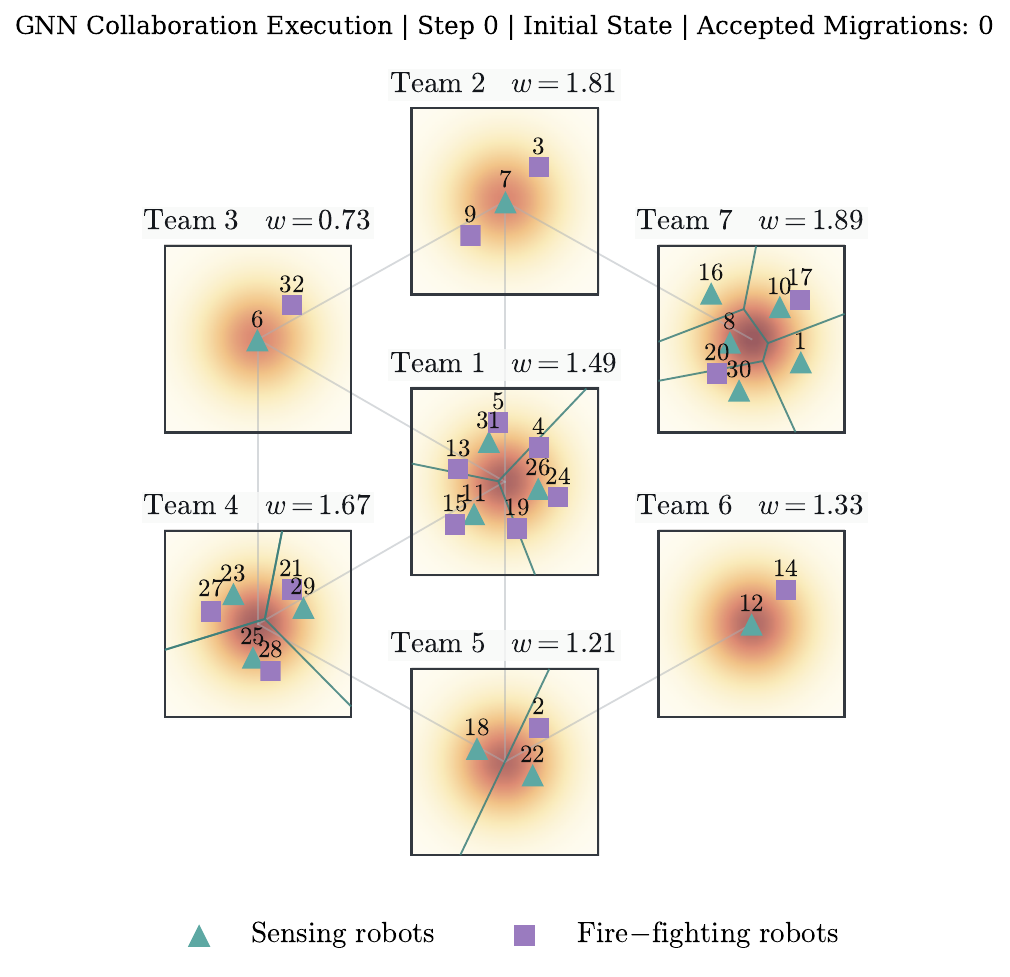}{7 teams, 32 robots\\Initial simulation frame}\label{fig:example_initial_7_teams}}\hfill
    \subfloat[Final Allocation]{\resultframe{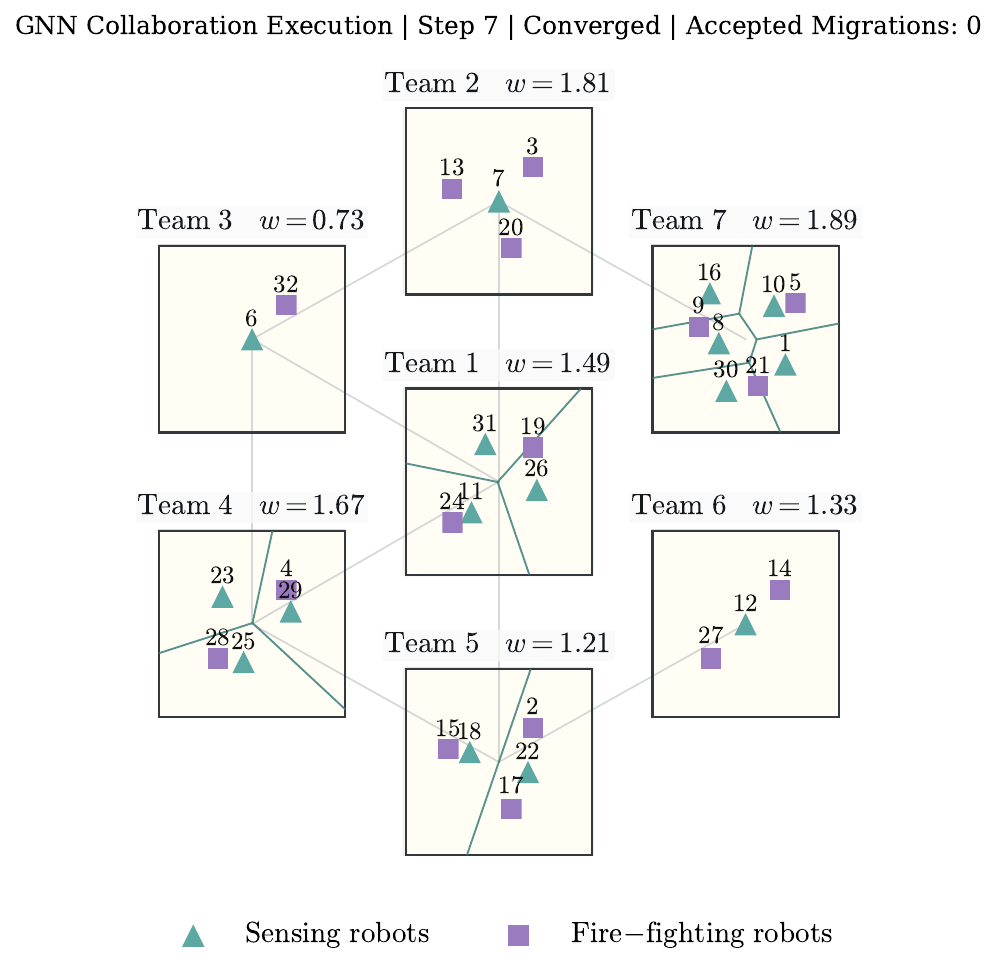}{7 teams, 32 robots\\Final simulation frame}\label{fig:example_final_7_teams}}\\
    \subfloat[Initial Allocation]{\resultframe{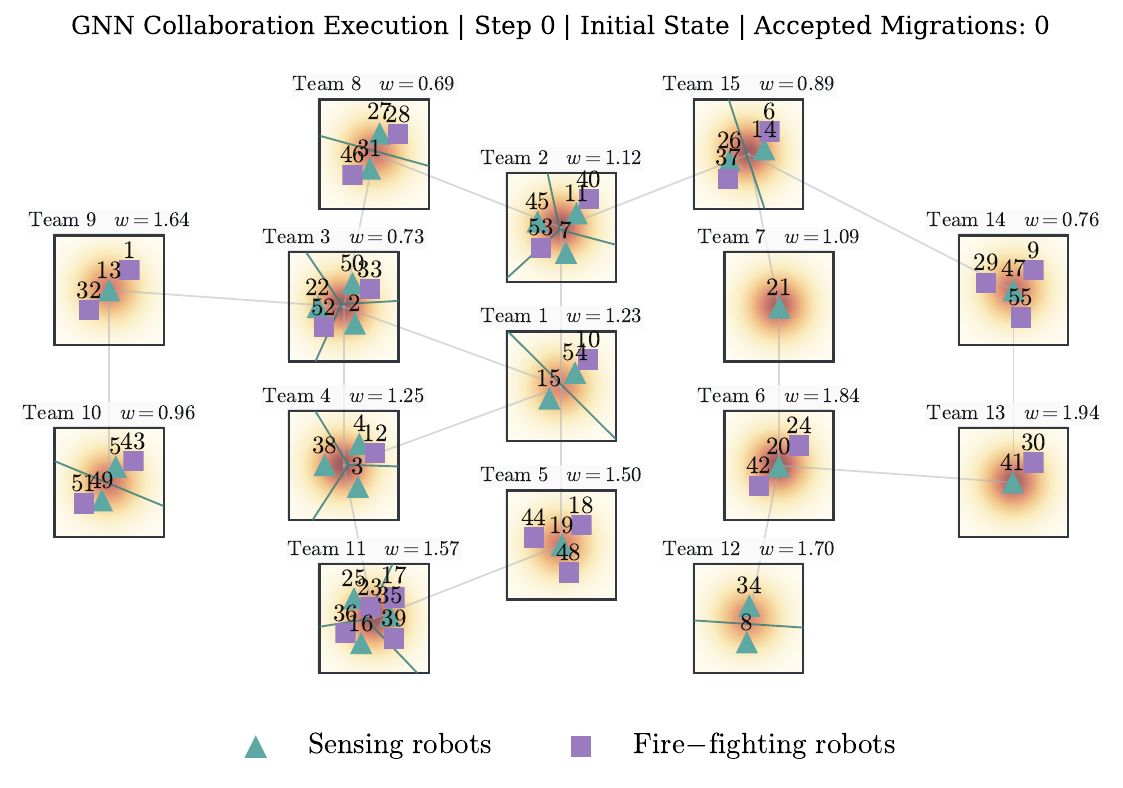}{15 teams, 55 robots\\Initial simulation frame}\label{fig:example_initial_15_teams}}\hfill
    \subfloat[Final Allocation]{\resultframe{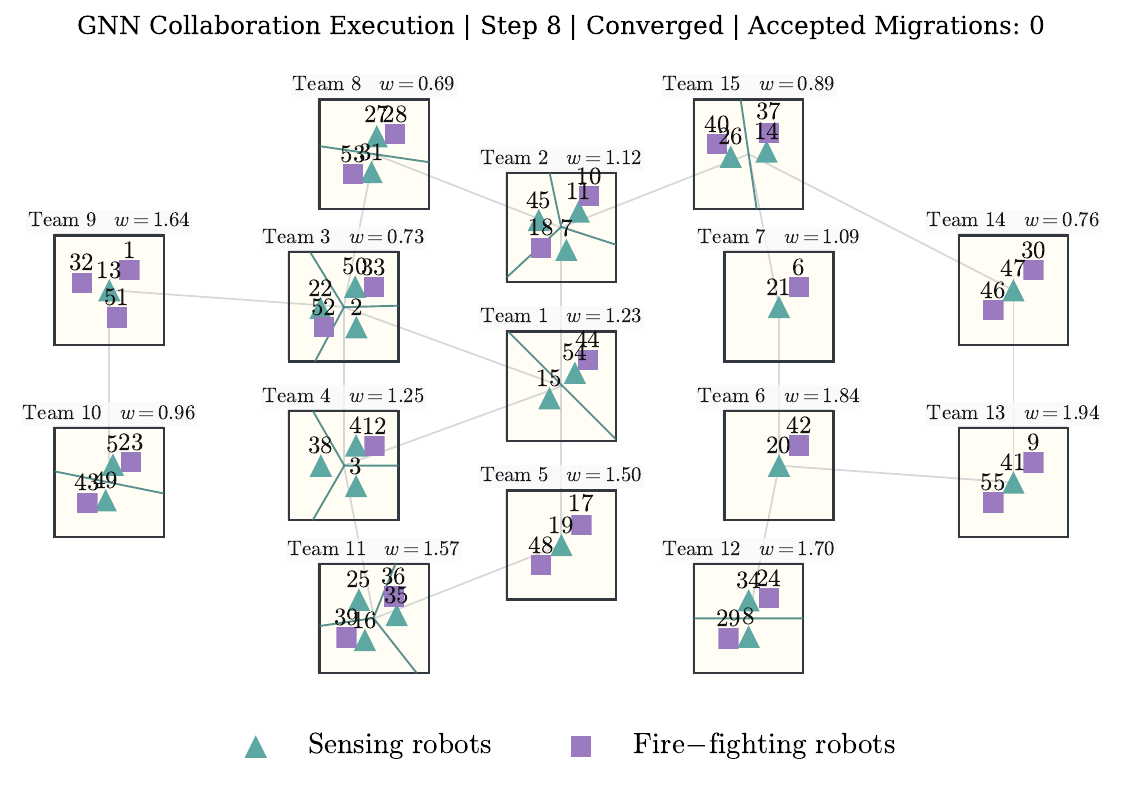}{15 teams, 55 robots\\Final simulation frame}\label{fig:example_final_15_teams}}
    \caption{
            Example executions of the proposed GNN-based heterogeneous altruistic collaboration framework in software simulations. The top row shows a system with \( 7 \) teams and \( 32 \) robots, while the bottom row shows a larger out-of-training-distribution system with \( 15 \) teams and \( 55 \) robots. Sensing robots and fire-fighting robots are represented by triangles and squares, respectively, and {the number above each robot identifies that robot}. Each team is assigned a weight \( w \) signifying its importance. Starting from the initial allocations, the learned GNN policy evaluates Hamilton-admissible transfers and iteratively executes decentralized robot migrations to improve collaborative fire suppression performance.
        }
    \label{fig:gnn_execution_examples}
\end{figure*}

{The GNN improves the objective in all $14$ examples and achieves a larger improvement than the search's returned allocation in $11$. Search yields a better objective at three and four teams, and both methods produce the same allocation at five teams. For six through ten teams, the GNN produces a better objective in less time than the first improving search solution. At every tested size from fifteen teams onward, search reaches the $600$-second cap, while the GNN produces an improvement in under $3$ s. At fifty teams, mission value improves by $16.84\%$, corresponding to $15.89\%$ after transfer costs. These examples support the policy's ability to produce useful allocations promptly.}

\subsection{Software Simulations} \label{ssec:simulation}

{We next assess how repeated GNN decisions affect mission completion through $50$-step rollouts from twelve initial states spanning $3$-$20$ teams and $6$-$60$ robots. Evaluating three trained GNNs from each state gives $36$ trajectories, with joint completion requiring at most $0.1\%$ of the initial weighted fire mass and five consecutive transfer-free steps. The GNN satisfies this criterion in $34$ of the $36$ trajectories, with a median of $16$ steps among completed runs. Across all trajectories, including those that do not satisfy the criterion, the mean weighted fire mass remaining at step $50$ is $0.4337\%$ of its initial value. These results connect the allocation decisions to sustained fire suppression over repeated interactions.}

{Figure~\ref{fig:gnn_execution_examples} illustrates how team assignments evolve in systems of seven teams with $32$ robots and fifteen teams with $55$ robots. Triangles represent sensing robots, squares represent fire-fighting robots, and arrows indicate accepted transfers. Starting from the initial assignments, the policy redistributes robots according to their capabilities and the current mission state. The seven-team example executes seven allocation decisions, while the fifteen-team example executes eight. Reallocation ends when no transfer is accepted, after which the robots continue working within their assigned regions.}

\subsection{Robotic Testbed Experiments} \label{ssec:robotrobot}

{We evaluated the GNN on a Robotarium testbed with $15$ robots---eight sensing and seven fire-fighting robots---assigned to four rectangular regions. Five initial conditions were each tested with the GNN and a static-allocation control, giving ten physical trials. Within each pair, both methods share the initial assignments, robot capabilities and capacities, team weights, and fire fields. The static control also uses the same within-team motion and density update but keeps membership fixed while all GNN trials use the trained policy.}

\begin{figure}[!t]
    \centering
    \includegraphics[width=\columnwidth]{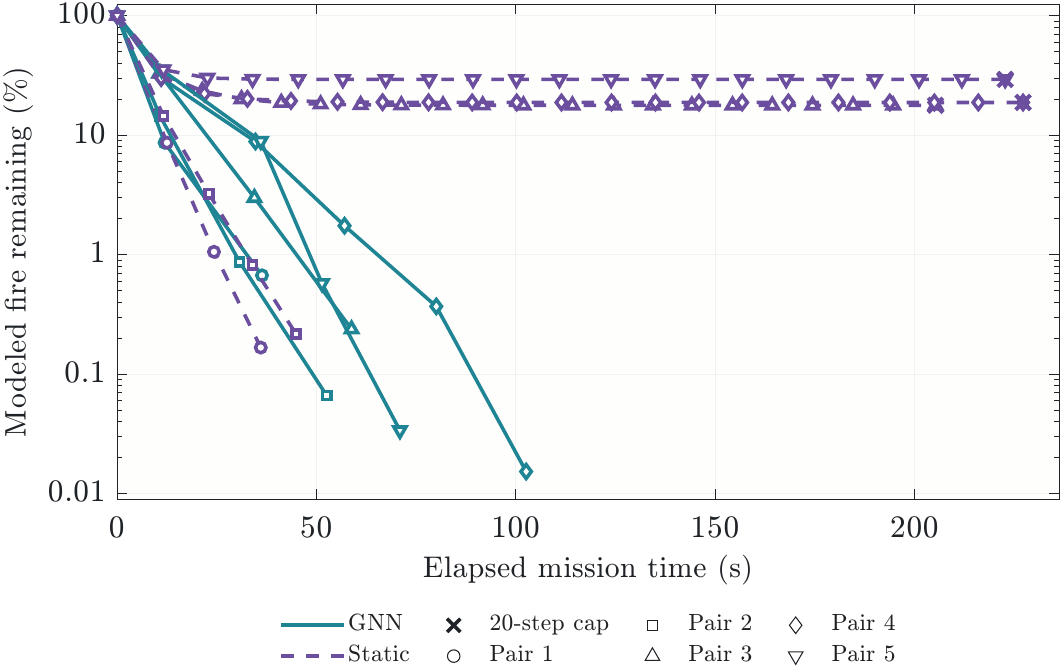}
    \caption{{Physical testbed results for all five paired initial conditions on a common elapsed-time axis, excluding initialization. Solid teal curves show the GNN and dashed purple curves show static allocation; matching marker shapes identify each pair. Markers show remaining modeled fire, $100R_k$, after completed work steps. Each curve ends at its recorded stopping point, with crosses marking the $20$-step cap. The stopping rule is pointwise, so terminal weighted residuals need not be equal. The fire field is numerical, while robot motion and elapsed time are recorded on the physical testbed.}}
    \label{fig:physical_trial_evolution}
\end{figure}

{As shown in Figure~\ref{fig:gnn_execution_robot}, the GNN sequence follows the robots through two migration phases to their final allocation, while the static sequence shows the initial and final states since no robot migrations are allowed. The overlays identify robot capabilities and team membership, while the fire colors represent the modeled density.}

{The testbed executes robot transfers and within-team coverage motion, while the fire density evolves according to Equation~\eqref{eq:phi_update_rule}. Density decay is determined by the model's canonical Lloyd coverage state, with measured robot poses and elapsed time characterizing physical execution separately. Each trial begins with $250$ controller iterations of work under the initial allocation; subsequent steps apply the policy on the testbed host, execute accepted transfers, and perform another $250$ controller iterations. Static allocation follows the same work schedule without transfers.}

\begin{table}[!t]
    \centering
    \footnotesize
    \setlength{\tabcolsep}{3pt}
    \renewcommand{\arraystretch}{1.2}
    \caption{{Five paired physical trials with $15$ robots and $4$ teams. $k_G$ is the GNN's final work step. Residual fire is compared at this same step in both runs; elapsed time is measured to each run's own stopping point. $20^\dagger$ denotes termination at the cap with the threshold unmet.}}
    \label{tab:physical_trials}
    {
    \begin{tabularx}{\columnwidth}{c|*{5}{>{\centering\arraybackslash}X|}>{\centering\arraybackslash}X}
        \hline
        & \multicolumn{2}{c|}{\shortstack{\textbf{Recorded}\\\textbf{work steps}}} & \multicolumn{2}{c|}{\shortstack{\textbf{Elapsed to}\\\textbf{stop (s)}}} & \multicolumn{2}{c}{\shortstack{\textbf{Residual at}\\\textbf{$k_G$ (\%)}}}\\
        \hline
        \textbf{Pair} & \textbf{GNN ($k_G$)} & \textbf{Static} & \textbf{GNN} & \textbf{Static} & \textbf{GNN} & \textbf{Static}\\
        \hline
        1 & 2 & 3 & 36.33 & 36.02 & 0.672 & 1.057\\
        2 & 3 & 4 & 52.56 & 44.86 & 0.066 & 0.818\\
        3 & 3 & $20^\dagger$ & 58.78 & 205.21 & 0.238 & 20.034\\
        4 & 5 & $20^\dagger$ & 102.56 & 227.16 & 0.015 & 18.999\\
        5 & 4 & $20^\dagger$ & 70.91 & 222.67 & 0.034 & 29.289\\
        \hline
    \end{tabularx}
    }
\end{table}

\begin{figure}[!b]
    \centering
    \includegraphics[width=\columnwidth]{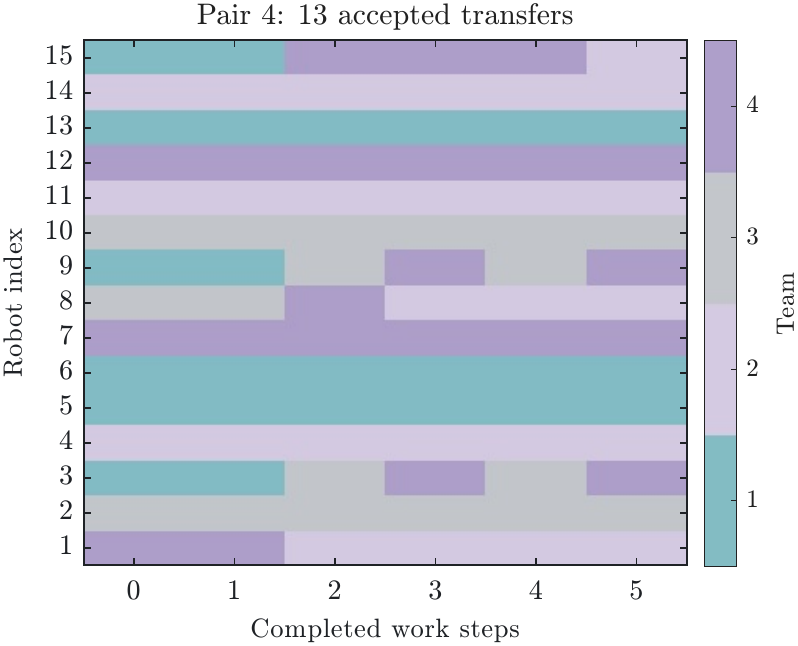}
    \caption{{Team membership during the GNN trial for pair $4$. Each row tracks one robot, and colors identify its assigned team after each completed work step; step $0$ shows the initial assignment. Color changes show the $13$ accepted transfers before work steps $2$-$5$.}}
    \label{fig:physical_trial_membership}
\end{figure}

{A trial ends after a work step if every value of the modeled fire field is below $1.5\%$ of its initial global peak, or when the $20$-step cap is reached. This pointwise threshold differs from the joint completion criterion used in the software simulations. To compare suppression across initial conditions, we normalize the remaining weighted fire mass after each completed work step by its initial value:}
{
\begin{equation} \label{eq:residual_fire}
    R_k=\frac{\sum_{v\in V} w_v\int_{\mathcal D_v}\phi_v^k(q)\,dq}
    {\sum_{v\in V} w_v\int_{\mathcal D_v}\phi_v^0(q)\,dq}.
\end{equation}
}
{Equation~\eqref{eq:residual_fire} integrates the density in each region and weights its contribution by the mission priority $w_v$. Normalization by the initial total gives $R_0=1$, so smaller values indicate less remaining modeled fire; results are reported as the percentage $100R_k$.}

\begin{figure*}[!t]
    \centering
    \subfloat[{GNN} Initial Allocation]{\resultframe{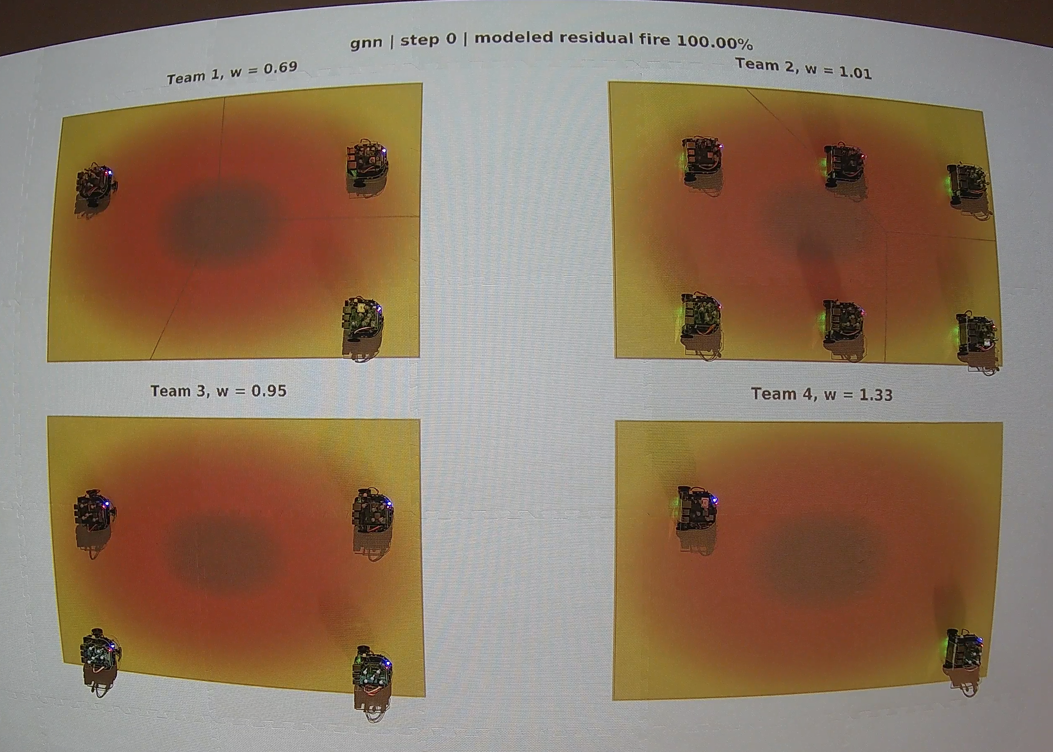}{GNN trial\\Initial video frame}\label{fig:initial}}\hfill
    \subfloat[{GNN} First Robot Migrations]{\resultframe{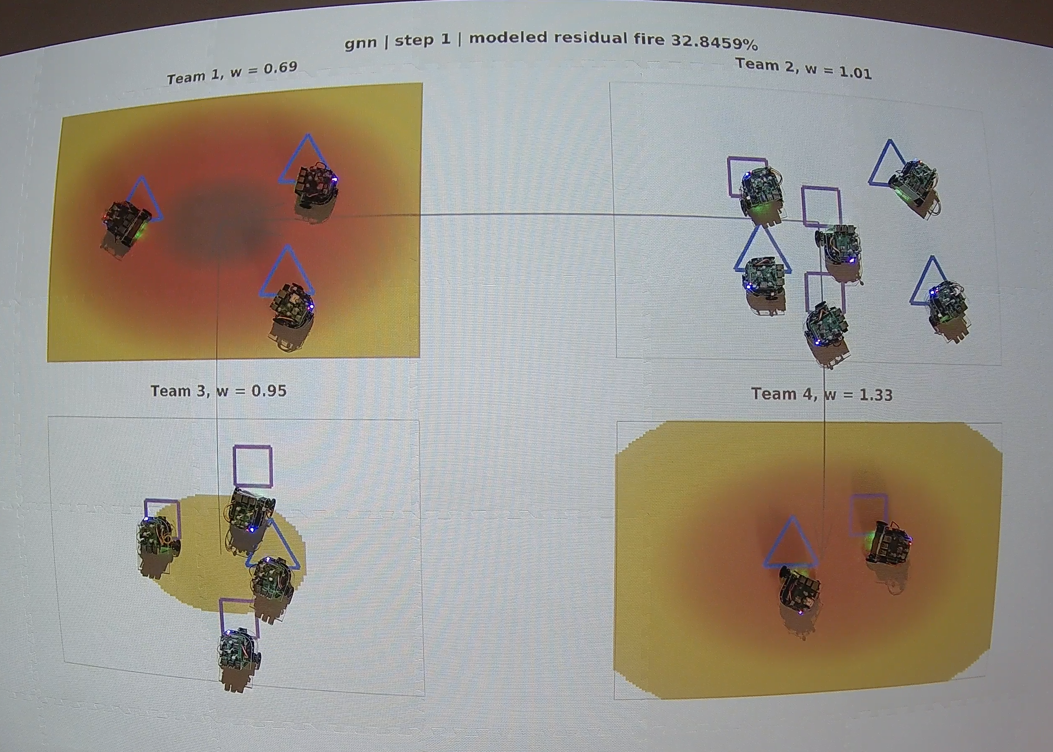}{GNN trial\\First migration frame}\label{fig:collab01}}\\
    \subfloat[{GNN} Second Robot Migrations]{\resultframe{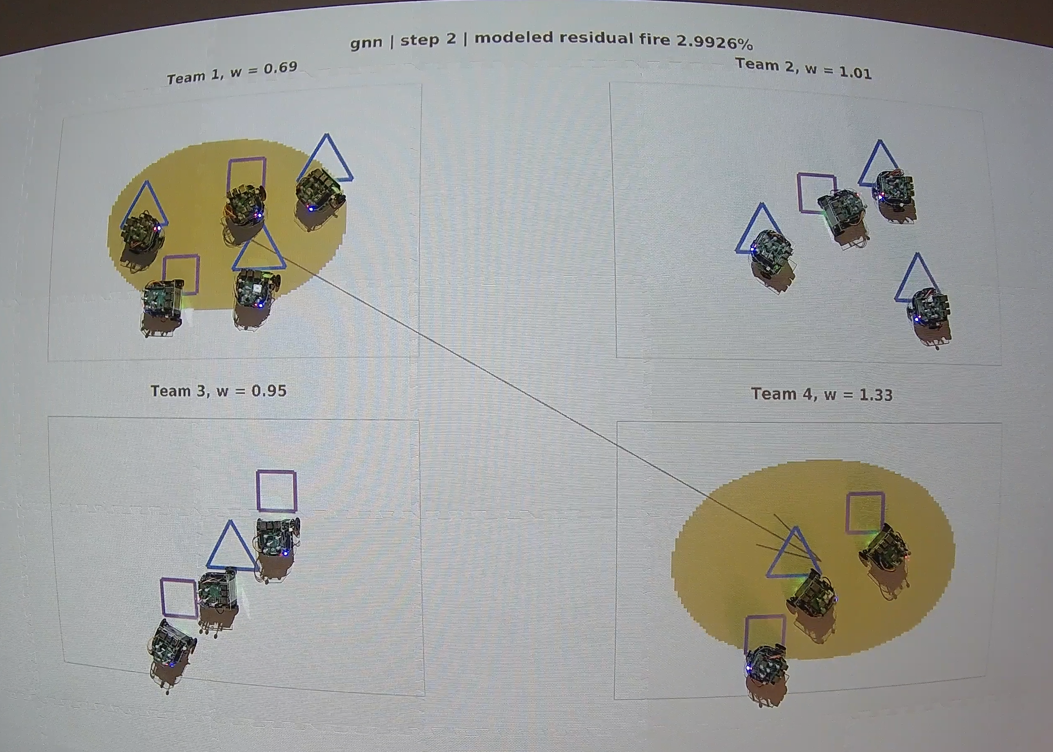}{GNN trial\\Second migration frame}\label{fig:collab02}}\hfill
    \subfloat[{GNN} Final Allocation]{\resultframe{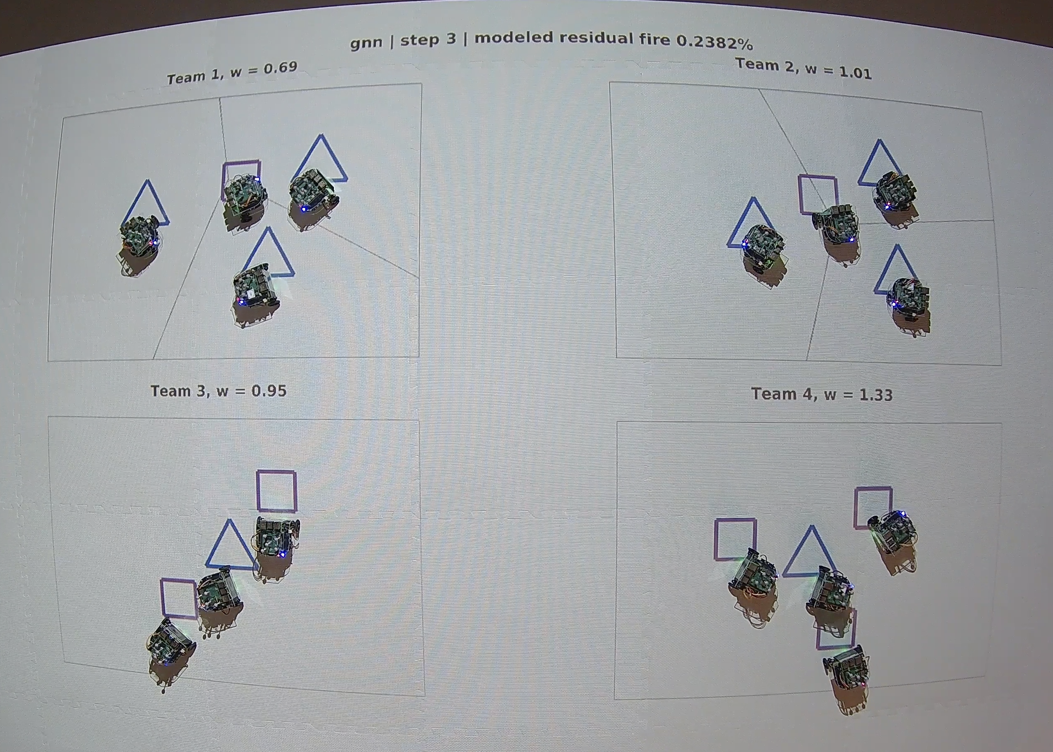}{GNN trial\\Final video frame}\label{fig:final}}\hfill
    \subfloat[{Static Initial Allocation}]{\resultframe{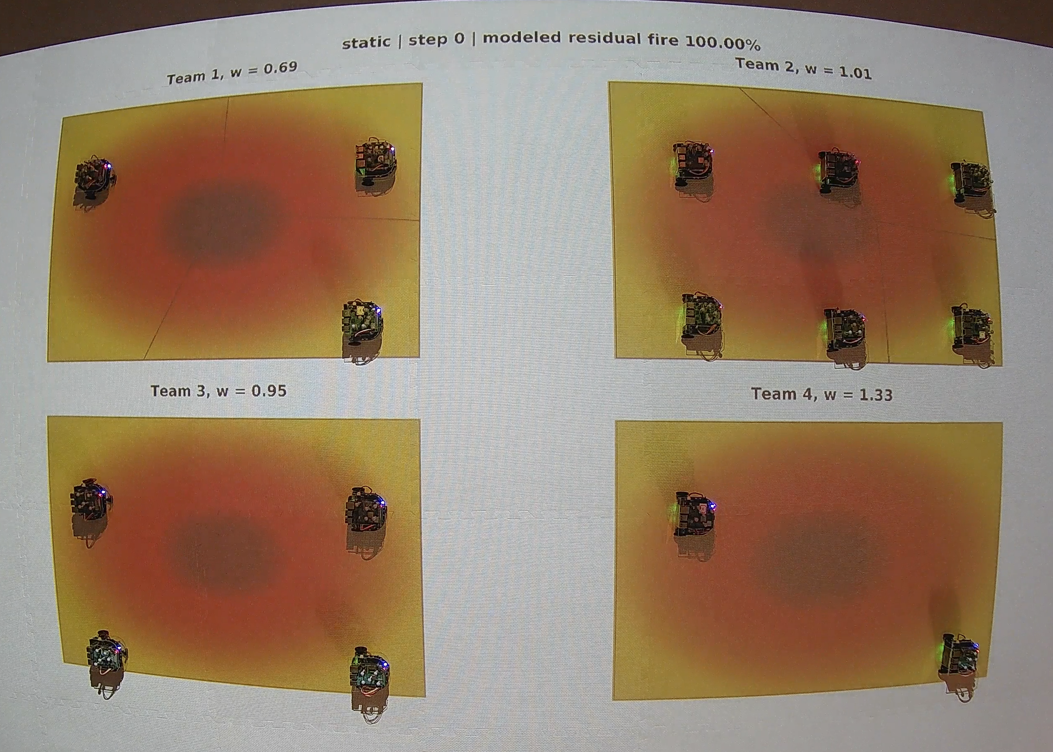}{Static control\\Initial video frame}\label{fig:initial2}}\hfill
    \subfloat[{Static Final Allocation}]{\resultframe{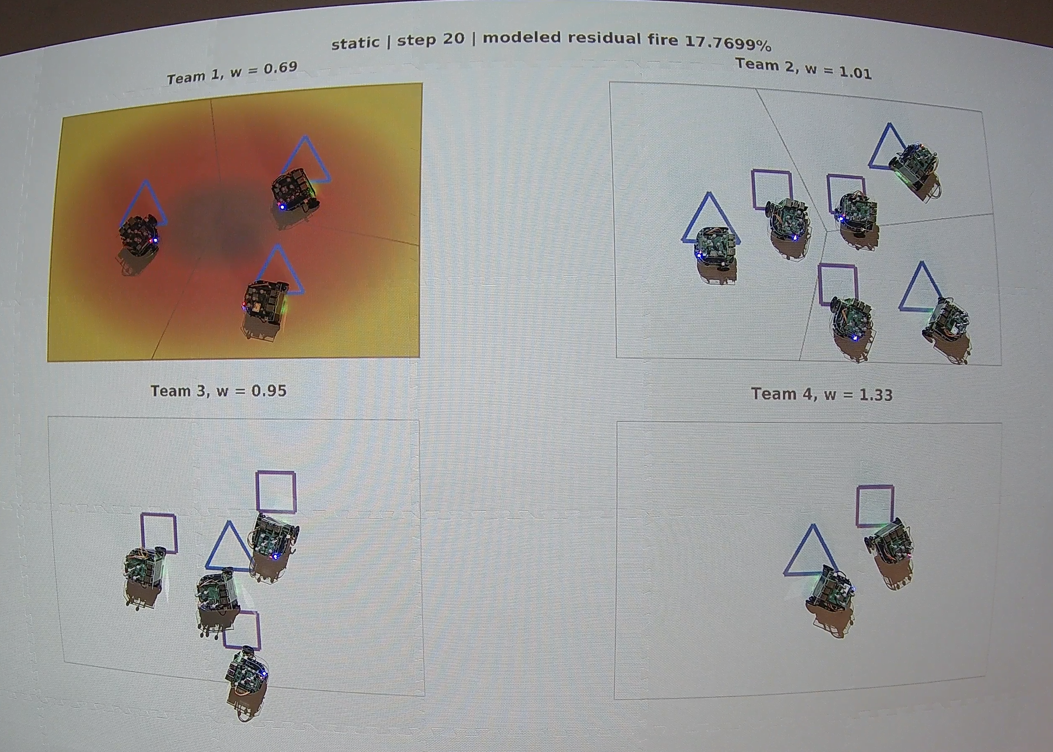}{Static control\\Final video frame}\label{fig:final2}}
    \caption{Execution sequence of the proposed GNN-based heterogeneous altruistic collaboration framework {and a static-allocation control} on a multi-robot testbed{, using the same initial conditions from pair $3$}. Teams operate in distinct spatial regions with heterogeneous sensing and fire-fighting robots, shown as triangles and squares, respectively. {Panels (a)-(d) show the GNN's initial allocation, two migration phases, and final allocation; panels (e)-(f) show the static control's initial and final states, with team membership held fixed. The weights of the top-left, top-right, bottom-left, and bottom-right teams are $0.69$, $1.01$, $0.95$, and $1.33$, respectively.}}
    \label{fig:gnn_execution_robot}
\end{figure*}

{Elapsed mission time begins after initial positioning and includes policy computation, migration, and within-team motion, with Figure~\ref{fig:physical_trial_evolution} tracing each run to its recorded stopping point. To complement these elapsed-time trajectories with a comparison at equal work-step counts, Table~\ref{tab:physical_trials} reports $100R_k$ at the GNN's stopping step, which is observed in both runs of every pair.}

{The GNN reaches the stopping threshold in all five trials within two to five work steps (median three), whereas static allocation reaches it in two trials, after three and four steps. At the GNN's stopping step, its residual fire is lower in every pair. The largest differences occur in pairs $3$-$5$, whose initial allocations leave one region with sensing robots but no fire-fighting capacity. Those regions cannot suppress their modeled fire without transfers, leaving the static control with $17.77\%$, $18.75\%$, and $29.27\%$ of the initial weighted fire mass at the $20$-step cap, respectively. Learned redistribution supplies the missing capacity and brings every region below the threshold in all three cases.}

{Across all five trials, the GNN executes $28$ transfers over twelve allocation calls, which take $75.6$ ms on average, including feature construction, Hamilton masking, inference, and filtering. Pair $4$ has the greatest reallocation activity, with $13$ transfers: $5$, $3$, $2$, and $3$ before work steps $2$-$5$, respectively; Figure~\ref{fig:physical_trial_membership} shows the resulting changes in each robot's team membership. Although redistribution restores suppression in initially unbalanced teams, its physical travel cost varies across conditions: in pairs 1 and 2, the GNN requires fewer work steps but more elapsed time than static allocation.}

\section{Conclusion} \label{sec:conclusion}

{This paper presented a framework for heterogeneous altruistic collaboration that combines capability-dependent mission values and transfer costs with GNN approximations of Hamilton-admissible one-step allocations. We established NP-hardness of the general assignment problem and evaluated the policy through supervised learning, larger-scale simulations, and physical trials. On the physical testbed, the GNN reached the modeled fire threshold in all five trials, compared with two for static allocation.}

\section*{Acknowledgment}
The authors want to thank Professor Yanning Shen and Alex Nguyen for helpful discussions.

\bibliographystyle{IEEEtran}
\bibliography{mybib}

\end{document}